\pdfoutput=1
\documentclass[11pt]{article}

\usepackage[table]{xcolor}
\PassOptionsToPackage{hyphens}{url}

\usepackage{booktabs}
\usepackage{multirow}
\usepackage{caption}
\usepackage{xspace}
\usepackage{afterpage}
\usepackage{microtype}
\usepackage{tikz}
\usetikzlibrary{arrows.meta,decorations.pathreplacing,calc,backgrounds,positioning}
\usepackage[numbers]{natbib}

\usepackage{amsmath,amsfonts,bm}

\def\eqref#1{equation~\ref{#1}}

\def\plaineqref#1{\ref{#1}}

\def\1{\bm{1}}

\def\rr{{\textnormal{r}}}

\DeclareMathAlphabet{\mathsfit}{\encodingdefault}{\sfdefault}{m}{sl}
\SetMathAlphabet{\mathsfit}{bold}{\encodingdefault}{\sfdefault}{bx}{n}

\usepackage{amsthm}

\usepackage[margin=1in, top=1in]{geometry}
\usepackage{graphicx}
\usepackage{fancyhdr}
\usepackage{hyperref}
\usepackage{xcolor}
\usepackage{titlesec}
\usepackage{tcolorbox}
\usepackage{titling}
\tcbuselibrary{skins}

\usepackage[T1]{fontenc}
\usepackage{tgheros}
\usepackage[utf8]{inputenc}

\usepackage{amsmath,amssymb,bm}
\usepackage{newtxtext}
\usepackage{newtxmath}

\AtBeginDocument{%

}

\providecommand{\titlefont}{\sffamily\bfseries}
\providecolor{qc_darkblue}{rgb}{0.008,0.063,0.247}
\providecolor{qc_blue}{rgb}{0.164,0.164,0.914}

\titleformat{\title}
{\titlefont\LARGE\color{qc_blue}}{}{0pt}{}

\titleformat{\section}
{\titlefont\Large\bfseries\color{qc_darkblue}}{\thesection}{1em}{}

\titlespacing*{\section}{0em}{1em}{.6em}

\newtcolorbox{titlebox}{
  enhanced,
  colback=white,
  boxrule=0pt,
  opacityback=0,
  opacityframe=0,
  width=0.95\textwidth,
  center
}

\makeatletter
\newcommand{\contactinfo}[1]{\def\@contactinfo{#1}}
\makeatother
\contactinfo{}

\renewcommand{\headrulewidth}{0.5pt}

\fancypagestyle{titlepage}{
  \fancyhf{}
  \fancyfoot[L]{\footnotesize Qualcomm AI Research is an initiative of Qualcomm Technologies, Inc.}
  \fancyfoot[R]{\footnotesize\thepage}
  \renewcommand{\headrulewidth}{0pt}
}
\renewcommand\abstract{%
    \setlength{\parskip}{.8em}
    \par
}

\title{Paper Title}
\date{February 23, 2024}
\author{Author1, Author2}
\contactinfo{\{author1, author2\}@qualcomm.com}

\renewenvironment{abstract}
  {\setlength{\parskip}{.8em}\par}
  {\par}

\newtheorem{proposition}{Proposition}
\newtheorem{lemma}{Lemma}
\newtheorem{corollary}{Corollary}
\newtheorem{assumption}{Assumption}
\newtheorem{remark}{Remark}

\newcommand{\name}{SyncVid\xspace}

\hypersetup{
  pdftitle={SyncVid: World-Equivariant Diffusion for Synchronized Video Generation},
  pdfauthor={Farhad G. Zanjani, Animesh Karnewar, Fatih Porikli, Amirhossein Habibian},
  colorlinks=true,
  linkcolor=qc_darkblue,
  citecolor=qc_darkblue,
  urlcolor=qc_blue,
}

\begin{document}
\thispagestyle{titlepage}
\pagestyle{fancy}
\bibliographystyle{unsrtnat}

\begin{tikzpicture}[remember picture,overlay]
  \node[anchor=north east, inner sep=0pt]
    at ([xshift=5cm,yshift=5.5cm]current page.north east)
    {\includegraphics[width=10cm,angle=-75,origin=c]%
      {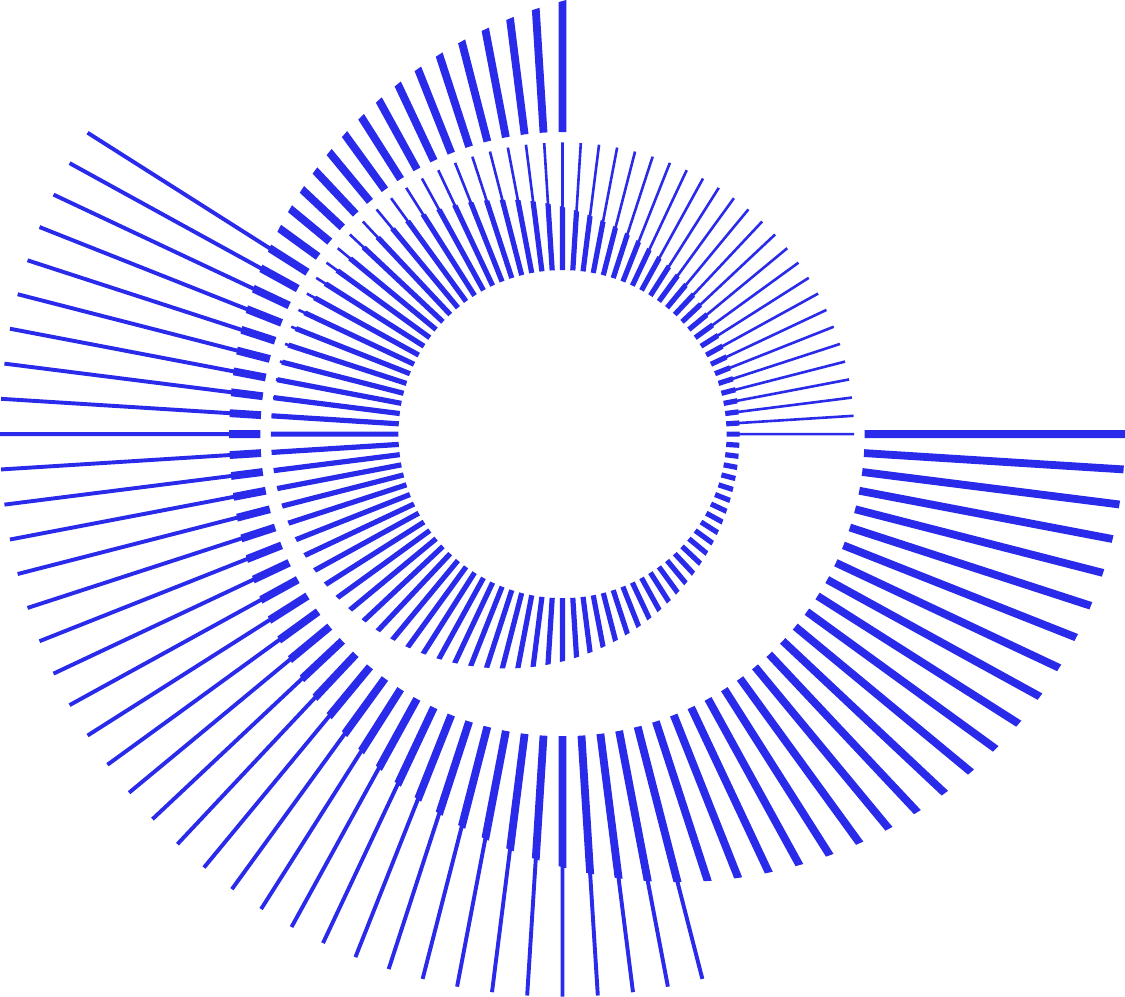}};
\end{tikzpicture}

\begin{figure}[t]
    \vspace*{-1cm}
    \hspace*{-0.6cm}
    \includegraphics[width=4.0cm]{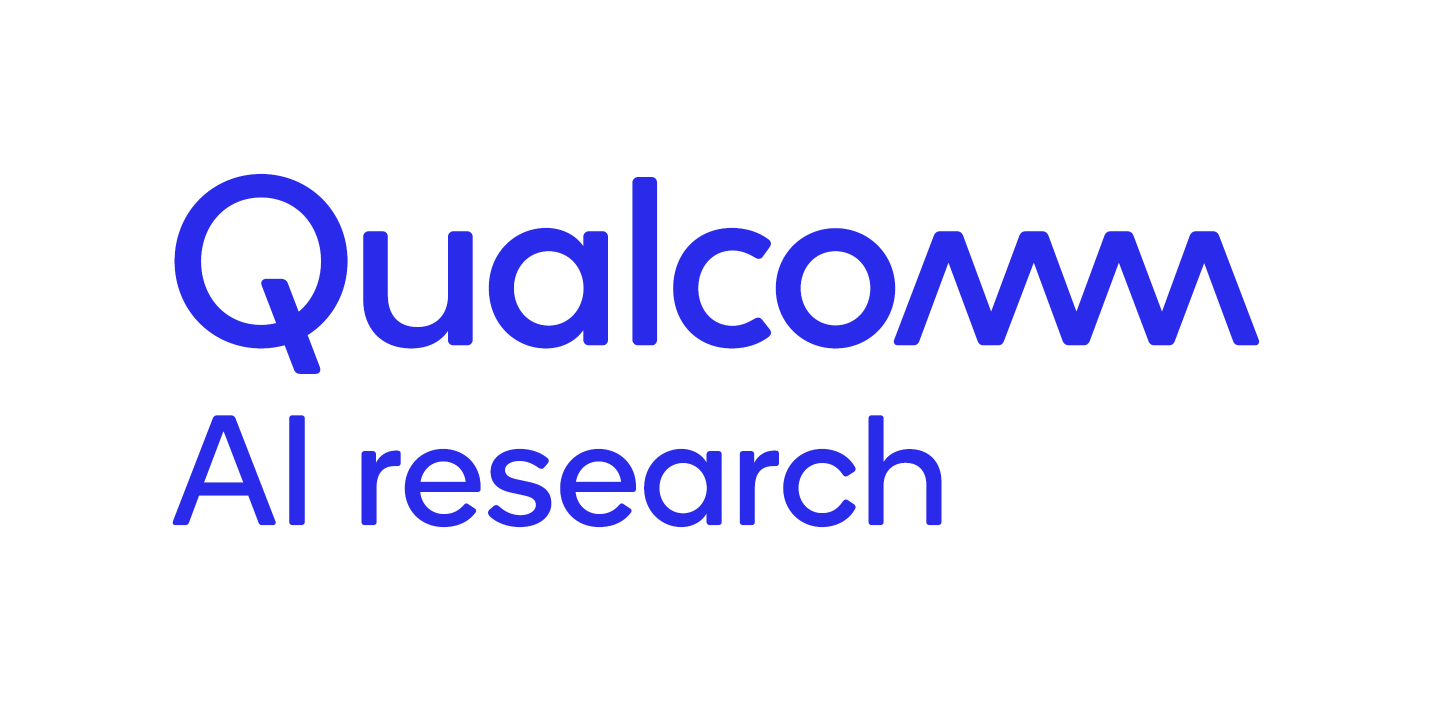}
    \vspace*{-0.5cm}
\end{figure}
\vspace{-2em}

\title{\textbf{\name}: World-Equivariant Diffusion for\\ Synchronized Video Generation}
\date{October 7, 2026}
\author{
    Farhad G.~Zanjani,
    Animesh Karnewar,
    Fatih Porikli,
    Amirhossein Habibian
}
\contactinfo{\{fzanjani, karnewar, fporikli, ahabibia\}@qti.qualcomm.com}

\begin{titlebox}
{\titlefont\huge\bfseries\color{qc_darkblue}\thetitle}\\

\makeatletter
{\titlefont\mdseries\normalfont\small\color{qc_blue}\@author\par
 \vspace{0.3em}
 {\bfseries\@contactinfo}%
}
\makeatother

\begin{abstract}
Generating multiple videos of a dynamic scene from a single reference video
requires synthesizing both observed and unobserved content consistently across
viewpoints. Existing camera-control methods for video diffusion typically rely
on architectural mechanisms, including camera embeddings, cross-view attention,
and geometric conditioning. Although effective for static scene content, these
approaches often struggle with the fine-grained, dynamic geometry that must
remain coherent across views and over time.
We synchronize the diffusion process itself by sampling the Gaussian noise prior
in a canonical 3D space derived from the source video and rasterizing it into
target cameras, rather than sampling independently in each 2D image plane. Our
method, \name, lifts a monocular video into a canonical 4D representation,
associates noise with scene geometry and unseen space, and rasterizes the
resulting structured noise into arbitrary views. Because rendering can
introduce local correlations, we further propose a cross-view-consistent
whitening procedure that decorrelates neighboring pixels while preserving
correspondence across views.
The structured noise serves as the diffusion prior for a LoRA-adapted
pretrained video diffusion model, requiring no architectural modifications. By
enforcing spatiotemporal and cross-view coherence at the source of stochasticity,
\name generates synchronized videos and improves both perceptual
quality and cross-view consistency in our experiments.
Project page: \url{https://qualcomm-ai-research.github.io/SyncVid-Webpage}.
\end{abstract}

\end{titlebox}

\afterpage{\begin{figure}[!t]
\centering
\includegraphics[
  width=\linewidth,
  trim=51pt 115pt 44pt 108pt,
  clip
]{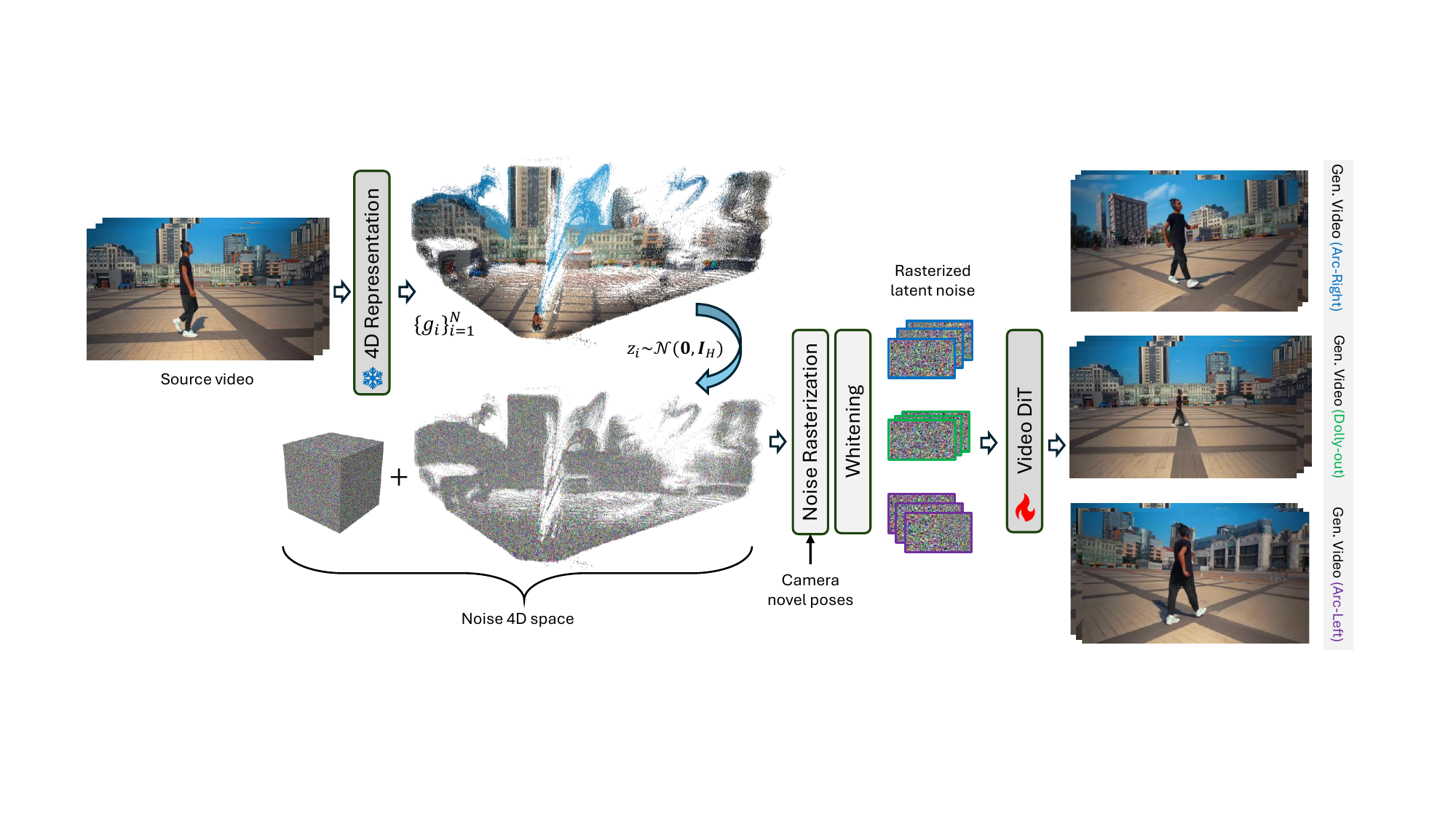}
\caption{Overview of SyncVid. Given a source video, SyncVid (1) lifts the
observations into a canonical 4D point representation; (2) samples a Gaussian
noise vector for each 3D primitive in the scene; and (3), after specifying
novel camera trajectories, rasterizes and whitens the shared 3D noise into
the corresponding views. These synchronized priors guide video diffusion to
produce temporally and cross-view consistent videos from multiple cameras.}
\label{fig:teaser}
\vspace{-8pt}
\end{figure}
}
\section{Introduction}
\label{sec:introduction}

Video diffusion models have rapidly matured from generating a single plausible
video to supporting fine-grained user control, including camera trajectories
\cite{he2025cameractrl,wang2024motionctrl,xu2024camco}.
A natural extension of this control is \emph{synchronized} video generation:
producing multiple videos of the same dynamic scene from different viewpoints
that remain consistent with one another.
This capability is central to applications such as virtual reshooting and
cinematography, multi-view content creation for AR/VR, and synthesizing
training data for downstream 3D and 4D vision tasks.
Recent work approaches this problem through camera-conditioned reshooting,
self-supervised pseudo-view training, or joint multi-view denoising with
cross-view interaction modules~\cite{bai2025recammaster,paliwal2026reshootanything,
bai2025syncammaster,kuang2024cvd,xu2025cavia}. These methods demonstrate that
video diffusion models can produce multiple views with useful visual and
geometric agreement, but rely primarily on architectural coupling, rendered
conditioning, or specialized training data.

A separate line of research has shown that the noise input to a video diffusion
model is itself a powerful conditioning signal, rather than merely an
initialization to be denoised away. Correlating noise across frames, instead of
sampling it independently, improves temporal consistency and enables direct
control over motion.
PYoCo~\cite{ge2023pyoco} introduces correlated noise priors for video
diffusion, while \emph{How I Warped Your Noise}~\cite{chang2024warpednoise}
and \emph{Go-with-the-Flow}~\cite{burgert2025gowiththeflow} warp a continuous
noise field according to optical flow. EquiVDM~\cite{liu2025equivdm} further
formalizes the resulting equivariance between noise transformations and
generated motion.
These results suggest that \emph{structured noise} -- noise with an explicit,
non-trivial correlation structure tied to the underlying content -- can serve as
an effective and lightweight conditioning mechanism, complementing changes in
model architecture or training data.

Despite this progress, current approaches for constructing structured noise
typically warp a 2D noise map using \textit{in-plane} rotations and translations
of 2D noise patches to propagate samples into future frames. Alternatively, they
require the optical flow of the target video, which is impractical to obtain
before generating that video.
These limitations can substantially degrade video quality, particularly for
dynamic objects undergoing \textit{out-of-plane} rotations, partial occlusions,
and subsequent reappearance.

In this work, we show that synchronizing the diffusion prior in a canonical 3D
space can induce world-equivariant video generation. Specifically,
camera-induced changes in the 3D noise representation are reflected as
consistent changes in the generated views. We achieve this by synchronizing
the diffusion process at its source of randomness: we construct a canonical 3D
noise representation that spans the reconstructed scene observed in the
reference video as well as its surrounding unobserved space. We then
rasterize this representation into each target camera to produce
view-consistent noise priors, including for complex transformations such as
out-of-plane rotations. Because all target views derive their stochastic
initialization from the same canonical representation, the diffusion model is
encouraged to generate correlated completions of unseen content, while
conditioning on the reconstructed scene preserves consistency in observed
regions. An overview of our method is shown in Fig.~\ref{fig:teaser}.
We summarize our contributions as follows:

\begin{itemize}
    \item We propose a structured noise prior for video diffusion models by
    sampling and assigning noise vectors in the 3D space of the video content,
    rather than independently in each camera's 2D image plane. Rasterizing this
    representation preserves dense spatiotemporal structure across views,
    without relying on 2D warping or editing of image-plane noise.
    \item We introduce a view-consistent whitening procedure that reduces local
    spatial correlations introduced by rasterization while preserving
    cross-view correspondence, keeping the resulting noise close to the
    marginal distribution expected by the pretrained diffusion model.
    \item We establish an automated benchmarking pipeline for synchronized video
    generation on the Mixkit subset of Open-Sora, enabling reliable and
    reproducible comparisons without manual or subjective selection of
    evaluation samples.
\end{itemize}

\section{Related Work}
\label{sec:related_work}

\subsection{Structured Noise for Synchronized Video Generation}
Structured noise has emerged as an effective mechanism for improving temporal
consistency and controlling motion in video diffusion models.
PYoCo~\cite{ge2023pyoco} observes that when an image diffusion model is adapted
into a video model, the noise injected at different frames should preserve temporal
correlation rather than being independent: independent per-frame noise lets the
denoiser drift toward a slightly different appearance in each frame, causing jitter
and flicker, whereas correlated noise starts neighboring frames from related latent
structure and denoises them into consistent appearance and smoother motion.
This gain requires no new architecture---no motion modules, optical-flow
supervision, or video transformers---but only a modified noise prior, so a
pretrained image model reaches strong temporal consistency by fine-tuning on
comparatively little video data at low cost.
Later work makes the correlation explicit: \emph{How I Warped Your Noise}~
\cite{chang2024warpednoise} transports a continuous Gaussian noise field across
frames while preserving its distribution, \emph{Go-with-the-Flow}~
\cite{burgert2025gowiththeflow} warps noise by real-time optical flow to control
object and camera motion, and EquiVDM~\cite{liu2025equivdm} shows that training
with warped noise encourages equivariance between noise transformations and
generated motion.
All build correspondence through 2D motion warping, which suits consecutive frames
of a single camera, where motion is well approximated by an in-plane flow field, but
does not extend across cameras: a 2D warp expresses neither out-of-plane camera
rotation nor correspondence for surfaces that a different viewpoint newly reveals
or occludes.

\subsection{Synchronized Camera-Controlled Video Generation}
Early camera-control methods such as CameraCtrl~\cite{he2025cameractrl},
MotionCtrl~\cite{wang2024motionctrl}, and CamCo~\cite{xu2024camco} condition
single-view generation on camera trajectories, but do not explicitly enforce
cross-view synchronization. Other methods couple denoising across views:
SynCamMaster~\cite{bai2025syncammaster} adds a multi-view synchronization module
for joint multi-camera generation, Collaborative Video Diffusion
(CVD)~\cite{kuang2024cvd} exchanges information through epipolar-guided
cross-view attention, and CAVIA~\cite{xu2025cavia} uses spatial and temporal
cross-attention for view consistency.
A third line uses geometry as an indirect synchronization signal:
ReCamMaster~\cite{bai2025recammaster} reshoots videos along novel trajectories
using video conditioning and synchronized synthetic multi-view training data,
Reshoot-Anything~\cite{paliwal2026reshootanything} reduces dependence on paired
multi-view data through self-supervised pseudo-view training around a 4D
point-cloud anchor, and ViewCrafter~\cite{yu2025viewcrafter} guides
diffusion-based novel-view video synthesis with coarse rendered geometry.
Diffusion-enhancement methods address related but different settings:
ArtiFixer~\cite{delutio2026artifixer} couples video diffusion with 3D
reconstruction, using opacity-aware latent mixing and autoregressive novel-view
generation to enhance incomplete reconstructions, while
SetDiff~\cite{zanjani2026setdiff} enhances Gaussian-Splatting novel-view
renderings in a multi-view setting by set-based diffusion over variable-size
reference and target views grounded in pixel-aligned coordinate maps and
pose-aware Plucker ray embeddings.
Overall, existing methods synchronize through camera embeddings, cross-view
interaction modules~\cite{bai2025syncammaster,kuang2024cvd,xu2025cavia}, or
rendered geometric conditions~\cite{bai2025recammaster,paliwal2026reshootanything,
yu2025viewcrafter,zanjani2026setdiff}. In contrast, we impose camera control with
structured noise: instead of sampling prior Gaussian noise independently in each
camera plane, we sample a canonical 3D field and project it into each view,
providing both geometric conditioning and view-consistent diffusion initialization.

\section{Method}

\subsection{Preliminary: Noise Sampling in Video Diffusion Models}
\label{sec:prelim_noise_sampling}

Let $\mathbf{x}_0 \in \mathbb{R}^{t \times C \times h \times w}$ denote a clean VAE latent 
features of a video $\mathbf{v} \in \mathbb{R}^{T \times 3 \times H \times W}$, where $C$ is 
the channel dimension of the VAE latent, and  $t, h, w$ are it's compressed spatiotemporal resolutions. 
A video diffusion model defines a forward noising process:
\begin{equation}
     \mathbf{x}_t = \sqrt{\bar{\alpha}_t}\mathbf{x}_0 +
    \sqrt{1-\bar{\alpha}_t}\boldsymbol{\epsilon}, \qquad \boldsymbol{\epsilon} \sim
    \mathcal{N}(\mathbf{0},\mathbf{I}), \label{eq:forward_diffusion}
\end{equation}
where $\bar{\alpha}_t$ is determined by the noise schedule. The denoising network is
commonly trained to predict either the sampled noise $\boldsymbol{\epsilon}$, the clean
latent $\mathbf{x}_0$, or a velocity parameterization derived from the two. The
standard formulation assumes that the entries of $\boldsymbol{\epsilon}$ are
independent and identically distributed (i.i.d.) standard Gaussian variables. In particular, for
two different spatiotemporal latent locations $p$ and $q$,
\begin{equation}
     \mathbb{E}[\boldsymbol{\epsilon}_p] = \mathbf{0}, \qquad
    \operatorname{Cov}(\boldsymbol{\epsilon}_p) = \mathbf{I}_C, \qquad
    \operatorname{Cov}(\boldsymbol{\epsilon}_p,\boldsymbol{\epsilon}_q) = \mathbf{0}
    \quad \text{for } p \neq q. \label{eq:white_noise_assumption}
\end{equation}
This i.i.d. noise assumption is part of the training distribution seen by the pretrained
diffusion model. Directly replacing the sampled noise with a strongly locally correlated noise pattern
can therefore introduce a train--test mismatch, even when the denoiser is adapted using
low-rank adaptation (LoRA). 

\subsection{Preliminary: Feed-Forward Dynamic Gaussian Splatting}
\label{sec:prelim_dynamic_gs}

Given source video frames
$\mathbf{I}^{\mathrm{src}}=\{\mathbf{I}^{\mathrm{src}}_t\}_{t=1}^{T}$, we employ an off-the-shelf
feed-forward dynamic Gaussian Splatting model e.g. NoPo4D~\cite{balice2026nopo4d} to
obtain a time-varying 3D representation:
\begin{equation}
     \mathcal{G},\, \boldsymbol{\Pi}^{\mathrm{src}}_{1:T} = \mathcal{F}_{\mathrm{GS}}
    \left( \mathbf{I}^{\mathrm{src}}_{1:T} \right),
\end{equation}
where $\boldsymbol{\Pi}^{\mathrm{src}}_t$ contains the estimated source-camera
parameters and
\begin{equation}
     \mathcal{G} = \{\mathcal{G}_i\}_{i=1}^{N}, \qquad \mathcal{G}_i = \left(
    \bar{\boldsymbol{\mu}}_i,\, \boldsymbol{\mu}_i(t),\, \boldsymbol{\Sigma}_i(t),\,
    \alpha_i,\, \mathbf{c}_i \right), \label{eq:primitive_attributes}
\end{equation}
denotes the set of 4D Gaussian primitives. Each primitive has time-varying geometry
and appearance together with a time-invariant, view-independent canonical anchor
$\bar{\boldsymbol{\mu}}_i\in\mathbb{R}^{3}$ used to identify it across motion and views.
For unobserved regions, we augment the reconstruction with a finite enclosing cubemap.
Implementation details of this 4D representation are provided in Appendix~\ref{sec:supp_experimental_details}.
Background samples are obtained by intersecting each target camera ray with this box,
which assigns every background pixel an explicit 3D surface point
\begin{equation}
     \mathbf{X}^{\mathrm{cube}}_p = \mathbf{o}_v + s_p \mathbf{d}_p, \qquad s_p = \min
    \left\{ s>0 : \mathbf{o}_v + s\mathbf{d}_p \in \partial \mathcal{B} \right\},
    \label{eq:cubemap_intersection}
\end{equation}
where $\mathbf{o}_v$ and $\mathbf{d}_p$ are the target-camera centre and ray direction.
Given a target pose $\boldsymbol{\Pi}^{\mathrm{tar}}_t$, the novel-view RGB frame is
obtained by compositing the rasterized Gaussian foreground with this background:
\begin{equation}
     \widetilde{\mathbf{I}}^{\mathrm{tar}}_t = \mathbf{C}^{\mathrm{GS}}_t +
    \left(1-\mathbf{A}^{\mathrm{GS}}_t\right) \odot \mathbf{C}^{\mathrm{cube}}_t,
    \label{eq:gs_cubemap_composition}
\end{equation}
where $(\mathbf{C}^{\mathrm{GS}}_t,\mathbf{A}^{\mathrm{GS}}_t)$ are the rasterized
foreground color and opacity. Rendering Eq.~\plaineqref{eq:gs_cubemap_composition}
over target poses produces the conditioning video used by our diffusion model.

\subsection{Canonical 3D Noise Latent} \label{sec:canonical_3d_noise}

\begingroup
\setlength{\abovedisplayskip}{3pt plus 1pt minus 1pt}
\setlength{\belowdisplayskip}{3pt plus 1pt minus 1pt}
\setlength{\abovedisplayshortskip}{2pt plus 1pt minus 1pt}
\setlength{\belowdisplayshortskip}{2pt plus 1pt minus 1pt}
\setlength{\jot}{1pt}

We represent the source video by the set of 4D
Gaussian primitives $\mathcal{G}=\{\mathcal{G}_i\}_{i=1}^{N}$ of
Eq.~\plaineqref{eq:primitive_attributes}, each carrying a canonical anchor position
$\bar{\boldsymbol{\mu}}_i\in\mathbb{R}^{3}$ alongside its time-varying Gaussian
Splatting attributes. We additionally assign each primitive an independent
high-dimensional noise vector
\begin{equation}
     \mathbf{z}_i \in \mathbb{R}^{H}, \qquad \mathbf{z}_i
    \overset{\mathrm{i.i.d.}}{\sim} \mathcal{N}(\mathbf{0},\mathbf{I}_H),
    \label{eq:primitive_noise}
\end{equation}
where $H$ denotes the dimension of the canonical 3D noise vectors and is chosen larger than the channel dimension $C$ of the VAE latent features, whose
typical values are $C\in\{16, 48\}$. The vectors $\{\mathbf{z}_i\}$ are sampled once
and shared across all source and target views. Crucially, $\mathbf{z}_i$ is attached to
the identity of primitive $i$ and is a function of neither time nor viewpoint: a
dynamic primitive whose center, covariance, or opacity varies over the sequence carries
the same $\mathbf{z}_i$ throughout. 

Unobserved regions are handled by the
same construction: we attach an i.i.d.\ noise vector
$\mathbf{z}^{\mathrm{cube}}_j\in\mathbb{R}^{H}$ to every texel $j$ of the enclosing box
of Sec.~\ref{sec:prelim_dynamic_gs}, and a background pixel reads the texel containing
its ray--box intersection $\mathbf{X}^{\mathrm{cube}}_p$. 

Given the target camera pose at time $t$, rasterizing all primitive noise features, contributed into pixel $p$ of 2D camera plane, gives the noise vector in camera view:
\begin{equation}
     \mathbf{Y}_p = \mathcal{R}_{\mathrm{GS}} (\mathcal{G}(t),\boldsymbol{\Pi}^{\mathrm{tar}}_t)=\sum_{i=1}^{N} w_{p,i}\mathbf{z}_i \in\mathbb{R}^{H},
    \label{eq:rasterized_noise_feature}
\end{equation}
where $w_{p,i}$ denotes the contribution of the $i$-th Gaussian point to pixel
$p$. The corresponding pixel $q$ in a different view that observes the same
scene region shares contributing primitives and therefore depends on the same
canonical noise vectors. In particular,
their cross-covariance is
\begin{equation}
     \operatorname{Cov}(\mathbf{Y}_p,\mathbf{Y}_q) = \left( \sum_{i=1}^{N}
    w_{p,i}w_{q,i} \right)\mathbf{I}_H. \label{eq:rasterized_cross_covariance}
\end{equation}
This shared dependence establishes cross-view noise consistency through the underlying
3D correspondence. The geometric correspondence underlying this result is proved
for arbitrary camera poses and time-dependent 4D deformations at a fixed timestamp (Appendix~\ref{sec:rotation_equivariance}).
Since $\mathbf{Y}_p$ is a linear combination of independent Gaussian
vectors, its marginal distribution and normalized feature are
\begin{align}
     \mathbf{Y}_p \sim \mathcal{N} \left( \mathbf{0}, v_p\mathbf{I}_H \right), \qquad
    v_p = \sum_{i=1}^{N} w_{p,i}^{2}. \label{eq:rasterized_noise_distribution}
\\[-1pt]
     \widetilde{\mathbf{Y}}_p = \frac{\mathbf{Y}_p}{\sqrt{v_p+\varepsilon}} = \frac{
    \sum_i w_{p,i}\mathbf{z}_i }{ \sqrt{\sum_i w_{p,i}^{2}+\varepsilon} },
    \label{eq:normalized_noise_feature}
\end{align}
where Eq.~\ref{eq:normalized_noise_feature} normalizes the rasterized feature using the variance factor $v_p$ 
and $\varepsilon>0$ provides numerical stability. The $\widetilde{\mathbf{Y}}_p$ has 
approximately zero mean and identity marginal covariance. Despite this normalization, as shown 
in Eq.~\plaineqref{eq:rasterized_cross_covariance}, nearby pixels sharing Gaussian primitives remain correlated:
\begin{equation}
     \operatorname{Cov} \left( \widetilde{\mathbf{Y}}_p, \widetilde{\mathbf{Y}}_q
    \right) = \rho_{pq}\mathbf{I}_H, \qquad \rho_{pq} = \frac{ \sum_i w_{p,i}w_{q,i} }{
    \sqrt{(v_p+\varepsilon)(v_q+\varepsilon)} }. \label{eq:normalized_cross_covariance}
\end{equation}
While this shared structure provides the desired \textbf{cross-view correspondence}, overlapping
Gaussian footprints also introduce undesirable local correlations within each rendered
noise map. Neighboring pixels typically share several of the same contributing
primitives, so their noise values remain far from independent even after the per-pixel
normalization above. Because this spatial correlation deviates from the i.i.d. Gaussian
distribution assumed by the pretrained diffusion model, it can degrade generation
quality when the model is only lightly adapted via fine-tuning (e.g. using LoRA). 
We address this next with a 3D-aware whitening procedure that reduces local spatial correlations 
from the rasterized noise while preserving its cross-view and temporal consistency.

\endgroup

\subsection{View-Consistent Whitening} \label{sec:view_consistent_whitening}

\providecommand{\wcBlobNoise}[3]{%
  \begin{scope}
    \clip (0,0) rectangle (#1,#2);
    \fill[black!12] (0,0) rectangle (#1,#2);
    \pgfmathsetseed{#3}%
    \foreach \i in {1,...,80}{%
      \pgfmathsetmacro{\xx}{rnd*#1}%
      \pgfmathsetmacro{\yy}{rnd*#2}%
      \pgfmathsetmacro{\rr}{0.07+rnd*0.13}%
      \pgfmathsetmacro{\gv}{int(8+rnd*80)}%
      \fill[black!\gv,opacity=0.8] (\xx,\yy) circle (\rr);%
    }
  \end{scope}
}

\providecommand{\wcWhiteNoise}[3]{%
  \begin{scope}
    \clip (0,0) rectangle (#1,#2);
    \pgfmathsetseed{#3}%
    \foreach \i in {0,...,29}{%
      \foreach \j in {0,...,21}{%
        \pgfmathsetmacro{\gv}{int(rnd*92)}%
        \fill[black!\gv] (\i*0.0667,\j*0.0625) rectangle ++(0.0672,0.063);%
      }
    }
  \end{scope}
}

\providecommand{\wcPosField}[2]{%
  \begin{scope}
    \clip (0,0) rectangle (#1,#2);
    \fill[white] (0,0) rectangle (#1,#2);
    \fill[cyan!45,opacity=0.85]    (0.22*#1,0.74*#2) circle (0.55*#2);
    \fill[magenta!40,opacity=0.75] (0.70*#1,0.62*#2) circle (0.50*#2);
    \fill[yellow!60,opacity=0.70]  (0.46*#1,0.20*#2) circle (0.46*#2);
    \fill[green!45,opacity=0.55]   (0.90*#1,0.16*#2) circle (0.34*#2);
  \end{scope}
}

\providecommand{\wcStack}[3]{%
  \foreach \k in {#3,...,1}{%
    \filldraw[fill=white,draw=black!45,line width=0.4pt]
      (\k*0.075,\k*0.075) rectangle ++(#1,#2);%
  }
}

\providecommand{\wcPixel}[3]{%
  \fill[orange!90!black] (#1,#2) circle (0.055);
  \draw[orange!90!black,line width=0.5pt] (#1,#2) circle (0.135);
  \node[font=\tiny\bfseries,text=orange!75!black,fill=white,
        fill opacity=0.85,text opacity=1,inner sep=0.5pt]
    at (#1+0.24,#2+0.20) {#3};
}

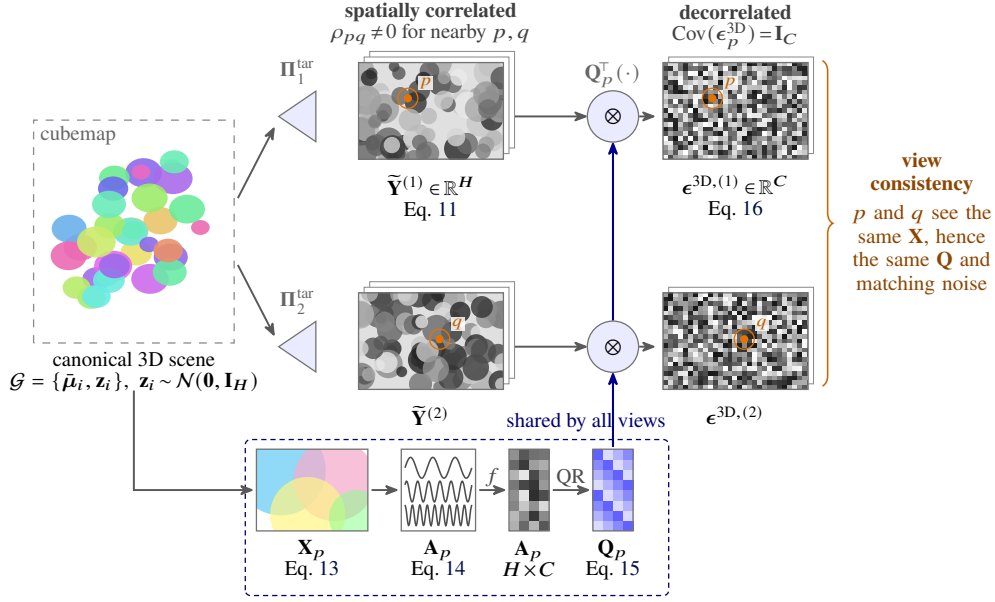
\begin{figure}[t]
\centering
\resizebox{.8\textwidth}{!}{%
\begin{tikzpicture}[
    font=\small,
    >={Stealth[round,length=2.2mm]},
    frame/.style={draw=black!55,line width=0.5pt},
    flow/.style={->,line width=0.75pt,black!65},
    share/.style={->,line width=0.8pt,blue!55!black},
    clab/.style={align=center,inner sep=1pt,font=\scriptsize},
    tag/.style={align=center,inner sep=1pt,font=\scriptsize,text=black!60},
  ]

  \draw[dashed,draw=black!45,line width=0.5pt] (0,-1.50) rectangle (2.70,1.50);
  \node[tag,anchor=north west] at (0.05,1.46) {cubemap};
  \pgfmathsetseed{21}
  \foreach \i in {1,...,34}{%
    \pgfmathsetmacro{\xx}{0.45+rnd*1.80}
    \pgfmathsetmacro{\yy}{-0.95+rnd*1.90}
    \pgfmathsetmacro{\rr}{0.10+rnd*0.12}
    \pgfmathsetmacro{\hh}{rnd}
    \definecolor{blobc}{hsb}{\hh,0.55,0.92}
    \fill[blobc,opacity=0.9] (\xx,\yy) ellipse (\rr*1.25 and \rr);
  }
  \node[clab,anchor=north] at (1.35,-1.58)
    {canonical 3D scene\\[-1pt]$\mathcal{G}=\{\bar{\boldsymbol{\mu}}_i,\mathbf{z}_i\}$,\ \
     $\mathbf{z}_i\!\sim\!\mathcal{N}(\mathbf{0},\mathbf{I}_H)$};

  \begin{scope}[shift={(3.30,1.55)}]
    \filldraw[fill=blue!8,frame] (0,0) -- (0.50,0.34) -- (0.50,-0.34) -- cycle;
    \node[tag,anchor=south] at (0.25,0.40) {$\boldsymbol{\Pi}^{\mathrm{tar}}_{1}$};
  \end{scope}
  \begin{scope}[shift={(3.30,-1.55)}]
    \filldraw[fill=blue!8,frame] (0,0) -- (0.50,0.34) -- (0.50,-0.34) -- cycle;
    \node[tag,anchor=south] at (0.25,0.40) {$\boldsymbol{\Pi}^{\mathrm{tar}}_{2}$};
  \end{scope}
  \draw[flow] (2.75, 0.45) -- (3.24, 1.40);
  \draw[flow] (2.75,-0.45) -- (3.24,-1.40);

  \begin{scope}[shift={(4.30,0.90)}]
    \wcStack{1.95}{1.30}{3}
    \begin{scope}[shift={(0.075,0.075)}]
      \wcBlobNoise{1.95}{1.30}{13}
      \draw[frame] (0,0) rectangle (1.95,1.30);
      \wcPixel{0.663}{0.819}{$p$}
    \end{scope}
    \node[clab,anchor=north] at (1.05,-0.10)
      {$\widetilde{\mathbf{Y}}^{(1)}\!\in\!\mathbb{R}^{H}$\\[-1pt]
       Eq.~\ref{eq:normalized_noise_feature}};
  \end{scope}
  \node[circle,frame,fill=blue!8,minimum size=0.62cm] (projA) at (7.80,1.55) {$\otimes$};
  \node[tag,anchor=south] at (7.80,1.90) {$\mathbf{Q}_p^{\!\top}(\cdot)$};
  \draw[flow] (6.48,1.55) -- (projA.west);
  \begin{scope}[shift={(8.40,0.90)}]
    \wcStack{1.95}{1.30}{2}
    \begin{scope}[shift={(0.075,0.075)}]
      \wcWhiteNoise{1.95}{1.30}{13}
      \draw[frame] (0,0) rectangle (1.95,1.30);
      \wcPixel{0.663}{0.819}{$p$}
    \end{scope}
    \node[clab,anchor=north] at (1.05,-0.10)
      {$\boldsymbol{\epsilon}^{\mathrm{3D},(1)}\!\in\!\mathbb{R}^{C}$\\[-1pt]
       Eq.~\ref{eq:final_structured_noise}};
  \end{scope}
  \draw[flow] (projA.east) -- (8.36,1.55);

  \begin{scope}[shift={(4.30,-2.20)}]
    \wcStack{1.95}{1.30}{3}
    \begin{scope}[shift={(0.075,0.075)}]
      \wcBlobNoise{1.95}{1.30}{29}
      \draw[frame] (0,0) rectangle (1.95,1.30);
      \wcPixel{1.092}{0.676}{$q$}
    \end{scope}
    \node[clab,anchor=north] at (1.05,-0.10) {$\widetilde{\mathbf{Y}}^{(2)}$};
  \end{scope}
  \node[circle,frame,fill=blue!8,minimum size=0.62cm] (projB) at (7.80,-1.55) {$\otimes$};
  \draw[flow] (6.48,-1.55) -- (projB.west);
  \begin{scope}[shift={(8.40,-2.20)}]
    \wcStack{1.95}{1.30}{2}
    \begin{scope}[shift={(0.075,0.075)}]
      \wcWhiteNoise{1.95}{1.30}{29}
      \draw[frame] (0,0) rectangle (1.95,1.30);
      \wcPixel{1.092}{0.676}{$q$}
    \end{scope}
    \node[clab,anchor=north] at (1.05,-0.10) {$\boldsymbol{\epsilon}^{\mathrm{3D},(2)}$};
  \end{scope}
  \draw[flow] (projB.east) -- (8.36,-1.55);

  \node[clab,anchor=south,text=black!75] at (5.35,2.48)
    {\textbf{spatially correlated}\\[-1pt]$\rho_{pq}\!\neq\!0$ for nearby $p,q$};
  \node[clab,anchor=south,text=black!75] at (9.45,2.42)
    {\textbf{decorrelated}\\[-1pt]$\operatorname{Cov}(\boldsymbol{\epsilon}^{\mathrm{3D}}_p)\!=\!\mathbf{I}_C$};

  \draw[rounded corners=2pt,draw=blue!45!black,line width=0.5pt,
        dash pattern=on 2pt off 1.5pt] (2.85,-4.90) rectangle (8.55,-2.80);
  \node[clab,anchor=south east,text=blue!45!black] at (8.55,-2.74)
    {shared by all views};

  \begin{scope}[shift={(3.00,-4.03)}]
    \wcPosField{1.50}{1.10}
    \draw[frame] (0,0) rectangle (1.50,1.10);
    \node[clab,anchor=north] at (0.75,-0.08)
      {$\mathbf{X}_p$\\[-1pt]Eq.~\ref{eq:canonical_position}};
  \end{scope}
  \begin{scope}[shift={(4.95,-4.03)}]
    \filldraw[fill=white,frame] (0,0) rectangle (1.00,1.10);
    \draw[black!70,line width=0.6pt,domain=0:1,samples=60,smooth,variable=\x]
      plot ({0.05+0.90*\x},{0.85+0.14*sin(\x*1080)});
    \draw[black!70,line width=0.6pt,domain=0:1,samples=60,smooth,variable=\x]
      plot ({0.05+0.90*\x},{0.53+0.14*sin(\x*2160+40)});
    \draw[black!70,line width=0.6pt,domain=0:1,samples=60,smooth,variable=\x]
      plot ({0.05+0.90*\x},{0.21+0.14*sin(\x*3240+90)});
    \node[clab,anchor=north] at (0.50,-0.08)
      {$\mathbf{A}_p$\\[-1pt]Eq.~\ref{eq:projection_preconditioning}};
  \end{scope}
  \begin{scope}[shift={(6.40,-4.03)}]
    \pgfmathsetseed{55}
    \foreach \i in {0,...,3}{\foreach \j in {0,...,7}{%
      \pgfmathsetmacro{\gv}{int(10+rnd*75)}
      \fill[black!\gv] (\i*0.135,\j*0.1375) rectangle ++(0.136,0.138);}}
    \draw[frame] (0,0) rectangle (0.54,1.10);
    \node[clab,anchor=north] at (0.27,-0.09) {$\mathbf{A}_p$\\[-1pt]$H\!\times\!C$};
  \end{scope}
  \begin{scope}[shift={(7.53,-4.03)}]
    \foreach \i in {0,...,3}{\foreach \j in {0,...,7}{%
      \pgfmathsetmacro{\gv}{int(14+16*mod(\i+\j,4))}
      \fill[blue!\gv!white] (\i*0.135,\j*0.1375) rectangle ++(0.136,0.138);}}
    \draw[frame] (0,0) rectangle (0.54,1.10);
    \node[clab,anchor=north] at (0.27,-0.08)
      {$\mathbf{Q}_p$\\[-1pt]Eq.~\ref{eq:projection_qr}};
  \end{scope}
  \draw[flow] (4.55,-3.48) -- (4.90,-3.48);
  \draw[flow] (6.00,-3.48) -- (6.35,-3.48)
    node[midway,above,font=\scriptsize,text=black!70,inner sep=1pt] {$f$};
  \draw[flow] (6.99,-3.48) -- (7.48,-3.48)
    node[midway,above,font=\scriptsize,text=black!70,inner sep=1pt] {QR};

  \draw[flow] (1.35,-2.14) -- (1.35,-3.48) -- (2.95,-3.48);
  \draw[share] (7.80,-2.90) -- (projB.south);
  \draw[share] (projB.north) -- (projA.south);

  \draw[decorate,decoration={brace,amplitude=4pt},draw=orange!80!black,
        line width=0.6pt] (10.62,2.30) -- (10.62,-2.08);
  \node[clab,anchor=west,text=orange!55!black,text width=2.15cm] at (10.84,0.11)
    {\textbf{view\\consistency}\\[2pt]
     $p$ and $q$ see the same $\mathbf{X}$, hence the same $\mathbf{Q}$
     and matching noise};

\end{tikzpicture}%
}
\caption{\textbf{View-consistent whitening}
Noise vectors $\mathbf{z}_i$ are attached to 3D primitives and the cubemap;
rasterization therefore provides cross-view correspondence, while overlapping
footprints induce local correlation. We reduce this correlation with a per-pixel
orthonormal transformation $\mathbf{Q}_p$ derived from the rasterized \emph{canonical
anchor} position $\mathbf{X}_p$ using random Fourier features and reduced QR
(dashed box). Because $\mathbf{Q}_p$ depends on $\mathbf{X}_p$ rather than
screen coordinates, pixels observing the same surface point share the transformation across views and over time. The resulting noise pattern (right) is approximately
i.i.d.\ Gaussian within each frame while remaining cross-view synchronized.}
\label{fig:view_consistent_whitening}
\vspace{-8pt}
\end{figure}

The rasterized noise from Sec.~\ref{sec:canonical_3d_noise} is locally correlated
because neighboring pixels may share several Gaussian primitives. We reduce this
correlation with a spatially varying projection indexed by canonical 3D position
rather than screen coordinates. Consequently, corresponding surface points in
different views and at different times use approximately the same projection,
while nearby pixels with distinct canonical positions can use different
projections.

To obtain the canonical 3D location of pixel $p$, we treat each canonical anchor position
$\bar{\boldsymbol{\mu}}_i$ as a rasterized attribute, in the same way that color is rasterized in
Eq.~\plaineqref{eq:gs_cubemap_composition}. This gives the canonical-coordinate map in screen space:
\begin{equation}
     \mathbf{X}_p = \frac{ \sum_{i=1}^{N} w_{p,i}\bar{\boldsymbol{\mu}}_i }{
    \sum_{i=1}^{N} w_{p,i}+\varepsilon } \in \mathbb{R}^{3}.
    \label{eq:canonical_position}
\end{equation}
The weights $w_{p,i}$ come from the deformed geometry at time $t$, so $\mathbf{X}_p$
describes the observed surface in canonical rather than time-varying world space.
Using $\boldsymbol{\mu}_i(t)$ would make a moving surface traverse the transformation
field and change its whitened noise over time; canonical anchors instead tie
$\mathbf{X}_p$ to primitive identity.

For regions outside the source-video coverage that contain no reconstructed primitives, we get
$\mathbf{X}_p$ from the ray--box intersection $\mathbf{X}^{\mathrm{cube}}_p$ defined in
Eq.~\plaineqref{eq:cubemap_intersection}; corresponding pixels therefore
approximately share the same canonical coordinate. 

The rasterized coordinate map from Eq.~\ref{eq:canonical_position} has only
three channels, representing 3D positions projected into the camera plane.
Since the rasterized noise usually has many more channels than three,
we construct a higher-dimensional representation of $\mathbf{X}_p$ using fixed Random Fourier Features (RFF)
~\cite{tancik2020fourier}. The implementation forms the projection frame directly
from fixed frequency vectors and per-output-channel phases:
\begin{equation}
     [\mathbf{A}_p]_{a,c} =
     \sqrt{\frac{2}{H}}\,
     \sin\!\left(2\pi s\,\boldsymbol{\omega}_a^\top\mathbf{X}_p
     +\phi_{a,c}\right), \qquad
     \mathbf{A}_p\in\mathbb{R}^{H\times C}, \label{eq:projection_preconditioning}
\end{equation}
where $\boldsymbol{\omega}_a\in\mathbb{R}^3$ are fixed frequency vectors,
$\phi\in[0,2\pi)^{H\times C}$ contains fixed phases, and $s$ is a scene-dependent
frequency scale. The resulting matrix provides multiple fixed sinusoidal
responses to the same canonical position. Because the frequencies and phases
are view- and time-independent, the resulting representation is consistent
across views and over time.
We compute the reduced QR factorization:
\begin{equation}
     \mathbf{A}_p = \mathbf{Q}_p\mathbf{R}_p, \qquad \mathbf{Q}_p^\top\mathbf{Q}_p =
    \mathbf{I}_C, \label{eq:projection_qr}
\end{equation}
and use $\mathbf{Q}_p$ as the local transformation which is applied on rasterized noise $\widetilde{\mathbf{Y}}_p$ (see Figure~\ref{fig:view_consistent_whitening}). The frequency scale controls the
trade-off between decorrelation and stability: higher frequencies improve local
decorrelation but increase sensitivity to reconstruction and rasterization errors.
When canonical positions are sufficiently separated relative to the Fourier
wavelengths, the high-dimensional feature ($H\gg C$) gives neighboring
projection subspaces room to become nearly orthogonal while retaining shared
3D structure. The local-decorrelation behavior of this transformation field is
formalized in Appendix~\ref{sec:whitening_proof}, under an idealized assumption
that distinct canonical anchors induce generic, approximately independent
projection bases. The final structured noise at pixel $p$ is defined as
\begin{equation}
     \boldsymbol{\epsilon}_p^{\mathrm{3D}} = \mathbf{Q}_p^\top \widetilde{\mathbf{Y}}_p
    = \frac{ \mathbf{Q}_p^\top \sum_i w_{p,i}\mathbf{z}_i }{ \sqrt{\sum_i
    w_{p,i}^{2}+\varepsilon} } \in \mathbb{R}^{C}. \label{eq:final_structured_noise}
\end{equation}
Since $\mathbf{Q}_p$ has orthonormal columns, the marginal covariance is approximately
\begin{equation}
     \operatorname{Cov} \left( \boldsymbol{\epsilon}_p^{\mathrm{3D}} \right) =
    \mathbf{Q}_p^\top \mathbf{Q}_p = \mathbf{I}_C, \label{eq:marginal_whitening}
\end{equation}
up to the numerical stabilization and rasterization approximations in
Eq.~\plaineqref{eq:normalized_noise_feature}. Each projected noise vector therefore has
the standard Gaussian marginal expected by the pretrained diffusion model. The
remaining cross-pixel covariance is
\begin{equation}
     \operatorname{Cov} \left( \boldsymbol{\epsilon}_p^{\mathrm{3D}},
    \boldsymbol{\epsilon}_q^{\mathrm{3D}} \right) =
    \rho_{pq}\mathbf{Q}_p^\top\mathbf{Q}_q. \label{eq:projected_cross_covariance}
\end{equation}
Eq.~\plaineqref{eq:projected_cross_covariance} separates the correlation into the
primitive-overlap term $\rho_{pq}$ and the projection term
$\mathbf{Q}_p^\top\mathbf{Q}_q$. A varying canonical-position field (and consequently varying $\mathbf{A}_p$) can make the
latter small for nearby pixels, whereas a constant or screen-space-dependent field
would not provide view-consistent decorrelation. The method is therefore
\emph{view-consistent approximate whitening}, not exact covariance whitening:
shared primitives still prevent guaranteed independence, but the marginal remains
close to the pretrained model's standard Gaussian distribution.
The construction provides three properties. \emph{View consistency} follows because
corresponding views use approximately the same anchor blend, projection, and mixture
of primitive noise vectors. \emph{Temporal consistency} follows because tracked
primitives retain their anchors and noise vectors, so the whitened noise follows
the object rather than the image plane. The correspondence is exact when the
same primitives dominate the rasterization; visibility changes or primitive
re-instantiation can still introduce gradual variation in the noise mixture and
projection. \emph{Local decorrelation} follows because neighboring pixels
generally use different transformations.

\subsection{Synchronized Video Generation} \label{sec:synchronized_video_generation}

We perform synchronized video generation by LoRA fine-tuning a pretrained I2V
diffusion model while replacing its i.i.d. Gaussian noise with the view-consistent
whitened rasterized noise (Eq.~\ref{eq:final_structured_noise}).
This preserves the pretrained model's standard-Gaussian marginal noise statistics
while injecting a shared 3D correspondence signal across views and frames.
The video Diffusion Transformer (DiT) uses the Rectified Flow
framework~\cite{esser2024scaling}, where noising follows straight paths between the
clean latent distribution and the structured noise distribution:
\begin{equation}
\begin{aligned}
     \mathbf{x}_{\tau} &= (1-\tau)\mathbf{x}_{0}
     + \tau\boldsymbol{\epsilon}^{\mathrm{3D}}, \\
     d\mathbf{x}_{\tau} &= \mathcal{D}_{\theta,\phi}\!\left(
     \operatorname{Concat}_{T}\left(\mathbf{x}_{\tau}, \mathbf{x}^{c}\right),
     \tau\right)d\tau
\end{aligned}
\label{eq:forward_velocity}
\end{equation}

Here $\mathcal{D}_{\theta,\phi}$ denotes the DiT with frozen pretrained parameters
$\theta$ and learnable LoRA parameters $\phi$. The tensors
$\mathbf{x}_{0}, \boldsymbol{\epsilon}^{\mathrm{3D}} \in
\mathbb{R}^{t \times C \times h \times w}$ are the clean target-video VAE latent
and its associated rasterized noise, respectively. The conditioning
tokens $\mathbf{x}^{c}$ contain the source-video latent, the rendered RGB-video
latent, or both, and are temporally concatenated to the target tokens with extended
rotary position encoding. After the DiT blocks, the conditioning tokens are discarded
and the output head is applied only to the noised target tokens. Thus, the
conditioning stream provides context but is never denoised and never enters the
training objective. Only the LoRA parameters $\phi$ are optimized.

\label{sec:method}

\section{Experiments}
\label{sec:experiments}

\subsection{Datasets} \label{sec:datasets}
Because real-world datasets rarely provide synchronized videos of dynamic scenes, we
train on synthetic dynamic videos and real-world static videos. For synthetic
training, we use 36K video pairs from MultiCamVideo (MCV), introduced by
ReCamMaster~\cite{bai2025recammaster}, which provides synchronized videos and
camera trajectories rendered from dynamic scenes in Unreal Engine~5. To reduce the
synthetic-to-real domain gap, we additionally sample 2K DL3DV-10K
videos~\cite{ling2024dl3dv} for real-world training.

DL3DV contains monocular videos of static scenes rather than synchronized
multi-view captures, so we construct pseudo-synchronized pairs via temporal
folding, following CVD~\cite{kuang2024cvd}. For each $2l-1$-frame clip, we split
at the center frame ($l+1$) and reverse the first segment, yielding two segments
that share their starting frame and then follow different temporal trajectories.
Since the scenes are static, this reversal is physically valid, and each folded
pair provides two temporally aligned observations of the same scene along distinct
trajectories, with separate associated camera poses. Applying this procedure to
the first two buckets of DL3DV yields about 13.5K video pairs. 

\textbf{Multi-ViewBench.} For evaluation on real-world videos we curate \emph{Multi-ViewBench}, 
a 200-clip benchmark drawn from the Mixkit subset of Open-Sora-Plan~\cite{lin2024opensoraplan}. 
Because reshooting a real video admits no ground-truth target view, the composition of such a set 
is easy to bias by hand and hard for others to reproduce, so we select it automatically: 
a vision-language model judges every candidate clip once against six fixed admissibility 
gates and four semantic attributes, and every subsequent decision is a deterministic function of 
those recorded judgements. Admissible clips must be a single continuous real-time shot, landscape, 
and captured by a physically moving camera; survivors are ranked by a \emph{richness} index rewarding 
depth layering, occlusion, and non-rigid motion. We report two disjoint splits of 100 clips each, 
defined by the model's composition judgement: \emph{standard} holds the richest single-subject scenes, 
and \emph{complex} the richest multi-subject scenes with cluttered geometry, drawn from a category-balanced 
pool so that no Mixkit category monopolizes the candidates. The complex split is the harder of the two 
by construction, with a mean richness of 71.2 against 30.9 for standard. We release the code, 
the per-clip judgements, and the frozen split CSVs; Sec.~\ref{sec:eval_curation} documents the pipeline, 
the per-stage counts, and its limitations in full.

All training and evaluation videos are standardized to 5 sec. clips at 15\,fps and a resolution of 480P.

\subsection{Metrics} \label{sec:metrics}
\vspace{-3pt}
We evaluate cross-view synchronization, visual quality, camera accuracy, and perceptual consistency.
Synchronization uses high-confidence matched pixels in corresponding regions (Mat.\ Pix.),
cross-view FVD (FVD-V), and cross-view CLIP similarity (CLIP-V). Temporal consistency uses frame 
consistency (CLIP-F); following ReCamMaster~\cite{bai2025recammaster},
camera accuracy is measured by rotation and translation errors (RotErr/TransErr) between the
conditioning and recovered trajectories. We also report six no-reference VBench~\cite{huang2024vbench}
scores for aesthetic and imaging quality, temporal flickering and motion smoothness, and
subject/background consistency. Finally, a blind pairwise study presents the source and two
methods' renderings of the same target trajectory in randomized left/right order; participants
select the better result or no preference, covering 50 / 200 comparisons on two disjoint sets of
25 \textbf{Multi-ViewBench} scenes (SyncVid vs. Artifixer and SyncVid vs. Reshoot-Anything).

\subsection{Main results} \label{sec:results}
\vspace{-2pt}
\setlength{\floatsep}{8pt plus 1pt minus 1pt}
\setlength{\textfloatsep}{8pt plus 1pt minus 1pt}
\setlength{\intextsep}{8pt plus 1pt minus 1pt}

\begin{table}[t]
  \centering
  \caption{\textbf{Quantitative comparison on view synchronization, 
  temporal consistency and camera accuracy} on Multi-ViewBench videos. 
  Metrics are defined in Sec.~\ref{sec:metrics}. 
  }
  \label{tab:main_results_real}
  \resizebox{\textwidth}{!}{%
  \begin{tabular}{l|c|cccc|c|cc}
  \toprule
  & & \multicolumn{4}{c|}{View Synchronization} & Temporal Consistency & \multicolumn{2}{c}{Camera Accuracy} \\
  \cmidrule(lr){3-6} \cmidrule(lr){7-7} \cmidrule(lr){8-9}
  Method & \shortstack{Generation\\Length} & \multicolumn{2}{c}{Mat.\ Pix.\ $\uparrow$} &
  FVD-V $\downarrow$ & CLIP-V $\uparrow$ & CLIP-F $\uparrow$ & RotErr $\downarrow$ & TransErr $\downarrow$ \\
  \cmidrule(lr){3-4}
  & & Dynamic & Static & & & & & \\
  \midrule
  DaS~\cite{gu2025das}                               & 49 &  494 & 1251 & 315.69 & 86.01 & 97.40 & 0.1549 & 0.2164 \\
  Go-with-the-Flow~\cite{burgert2025gowiththeflow}   & 49 &  687 & 1277 & 296.30 & 87.44 & 96.78 & 0.1225 & 0.1109 \\
  ReCamMaster~\cite{bai2025recammaster}              & 75 & 1073 & 1428 & 233.70 & 88.06 & 97.39 & 0.1665 & 0.1778 \\
  ArtiFixer-14B~\cite{delutio2026artifixer}          & 75 & 1128 & 1937 & 195.55 & 88.99 & 97.55 & \textbf{0.0780} & \textbf{0.0529} \\
  Reshoot-Anything~\cite{paliwal2026reshootanything} & 69 & 1268 & \textbf{1989} & 202.22 & 89.96 & 97.82 & 0.1120 & 0.0838 \\
  \midrule
  \textbf{SyncVid (Ours)}                                      & \textbf{75} & \textbf{1332} & 1834 & \textbf{146.99} & \textbf{91.66} & \textbf{97.87} & 0.151 & 0.121 \\
  \bottomrule
  \end{tabular}}
\end{table}

\begin{table}[t]
  \centering
  \caption{\textbf{Quantitative comparison on VBench~\cite{huang2024vbench} metrics}.}
  \label{tab:main_results_vbench}
  \resizebox{\textwidth}{!}{%
  \begin{tabular}{lcccccc}
  \toprule
  Method & \shortstack{Aesthetic\\Quality $\uparrow$} & \shortstack{Imaging\\Quality $\uparrow$} &
  \shortstack{Temporal\\Flickering $\uparrow$} & \shortstack{Motion\\Smoothness $\uparrow$} & \shortstack{Subject\\Consistency
  $\uparrow$} & \shortstack{Background\\Consistency $\uparrow$} \\
  \midrule
  DaS~\cite{gu2025das}                               & 49.15 & 50.34 & 96.70 & 98.54 & 84.74 & 90.81 \\
  Go-with-the-Flow~\cite{burgert2025gowiththeflow}   & 49.91 & 51.98 & 95.63 & 98.00 & 84.87 & 90.69 \\
  ReCamMaster~\cite{bai2025recammaster}              & 48.12 & 49.51 & \textbf{97.22} & \textbf{98.89} & 85.41 & 89.91 \\
  ArtiFixer~\cite{delutio2026artifixer}              & 48.10 & 52.61 & 96.54 & 98.81 & 85.67 & 90.41 \\
  Reshoot-Anything~\cite{paliwal2026reshootanything} & 47.85 & 51.72 & 96.16 & 98.45 & 85.94 & \textbf{92.01} \\
  \midrule
  \textbf{SyncVid (Ours)} & \textbf{52.07} & \textbf{56.03} & 96.45 & 98.71 & \textbf{88.66} & 91.24 \\
  \bottomrule
  \end{tabular}}

\captionof{table}{\textbf{Human preference evaluations.} Percentage of blind pairwise comparisons.}
\label{tab:user-study}
\resizebox{0.9\linewidth}{!}{ \begin{tabular}{cccccc} \toprule SyncVid (Ours) & Tie & ArtiFixer~\cite{delutio2026artifixer} & SyncVid (Ours) & Tie & Reshoot-Any.~\cite{paliwal2026reshootanything} \\ \cmidrule(lr){1-3} \cmidrule(lr){4-6} \textbf{78}\% & 15\% & 7\% & \textbf{47}\% & 31\% & 22\% \\ \bottomrule \end{tabular} }
\end{table}

\vspace{-0.5em}
\begin{figure}[t]
\centering
\IfFileExists{floats/figures/qual_comparison/qual_comparison_mixkit-artistic-dance-of-a-woman-with-a-fire-saber-43667_arc_left.png}{%
  \includegraphics[width=\textwidth]{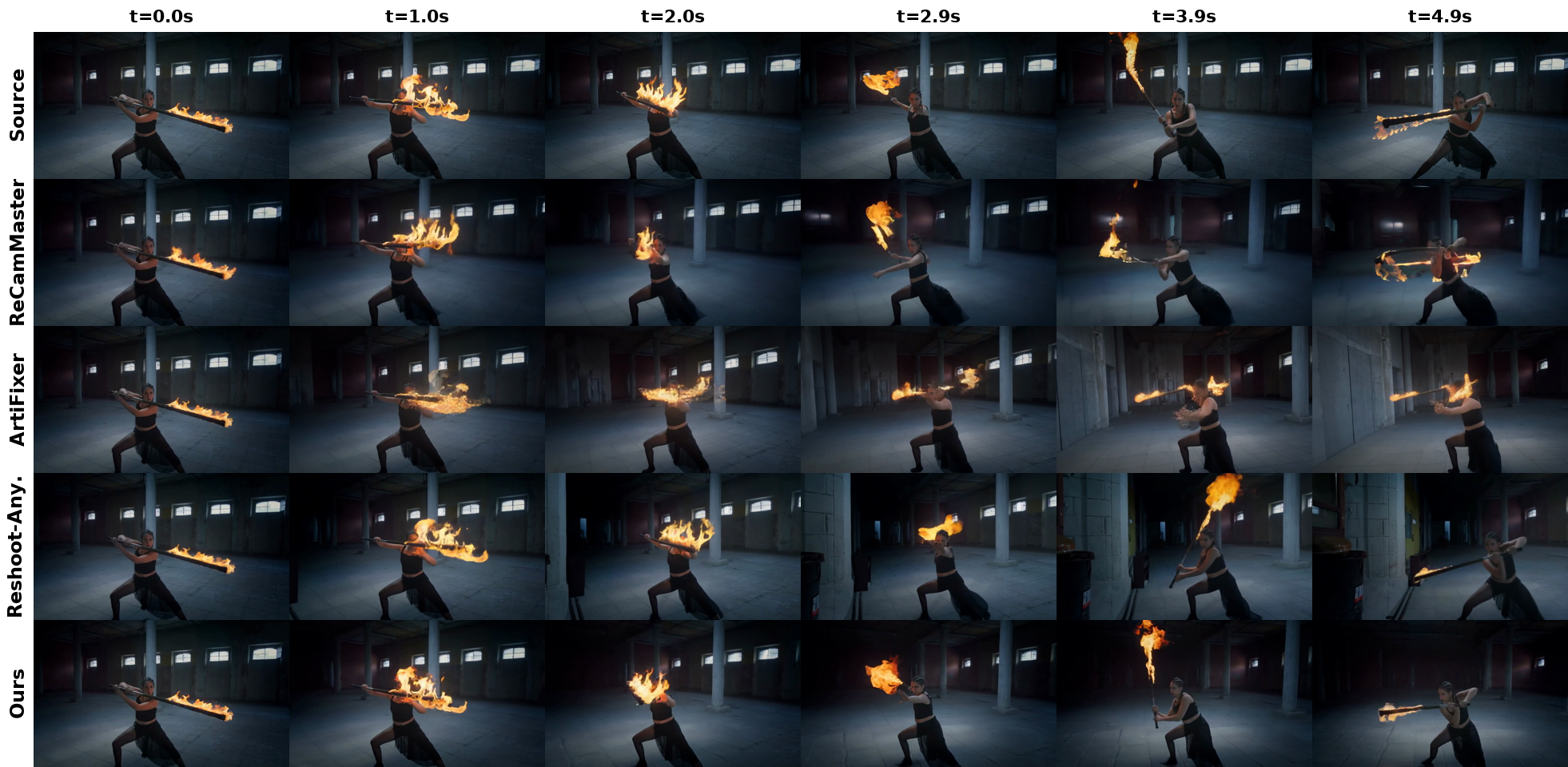}%
}{%
  \fbox{\parbox[c][3cm][c]{\textwidth}{\centering
    \texttt{floats/figures/qual\_comparison/qual\_comparison\_mixkit-artistic-dance-of-a-woman-with-a-fire-saber-43667\_arc\_left.png}
    not found --- run \texttt{make\_qual\_comparison.py} first.}}%
}
\vspace{-1.5em}
\caption{\textbf{Qualitative comparison} on a Multi-ViewBench example with an Arc-right trajectory.
}
\label{fig:qual_comparison}
\end{figure}

Following the evaluation protocol of ReCamMaster~\cite{bai2025recammaster}, we compare \name with recent state-of-the-art methods on the real-world Multi-ViewBench test set.
Table~\ref{tab:main_results_real} shows that our method provides the strongest overall
cross-view synchronization. It achieves the highest number of matched pixels in dynamic
regions (1,332) and the second-highest number in static regions (1,834), behind
Reshoot-Anything. It also obtains the lowest cross-view FVD (146.99) and the highest
cross-view CLIP similarity (91.66), indicating
that its generated views are both perceptually aligned and visually similar. In addition,
it achieves the highest temporal-consistency score (CLIP-F = 97.87). Camera accuracy is
not its main advantage: ArtiFixer obtains the lowest rotation and translation errors,
while our method remains competitive.

The VBench results in Table~\ref{tab:main_results_vbench} show a similar advantage in
perceptual quality. Our method achieves the highest aesthetic quality (52.07), imaging
quality (56.03), and subject consistency (88.66). Its temporal-flickering and
motion-smoothness scores are close to the best reported values, and its background
consistency (91.24) is second only to Reshoot-Anything. Overall, these results suggest
that the proposed view-consistent noise
improves cross-view alignment and content fidelity without substantially compromising
temporal or background stability. Table~\ref{tab:user-study} shows the user study results.

\subsection{Ablations} \label{sec:ablations}
\vspace{-2pt}
We ablate the main design choices of our method on a held-out set of OpenSora
videos, using the lighter Wan2.2-5B backbone to keep the study computationally
affordable. Results are reported in Table~\ref{tab:ablation_components}.
Whitening improves all four reported metrics over both standard noise and
unwhitened noise: it increases Mat.~Pix. from 222 to 267, reduces FVD-V from
518.77 to 461.44, and raises CLIP-V from 88.76 to 91.05, while slightly
improving CLIP-F. The conditioning ablation reveals a complementary trade-off.
Adding the source view yields the strongest direct synchronization, with the
highest Mat.~Pix. (319), CLIP-V (91.41), and CLIP-F (97.95). Using both the
rendered and source views instead gives the lowest FVD-V (461.44), suggesting
that the additional rendered view improves distributional visual similarity
even though it does not maximize pixel-level matching. We favor this configuration
because FVD-V, which compares the distributions of generated and reference views,
better reflects the overall quality of the generated videos than pixel-level matching.

LoRA~256 is the best rank on all four metrics, and the larger
DiT improves almost every metric. Our final model therefore uses whitening, joint
source- and rendered-view conditioning, LoRA~256, and the larger backbone.

\begin{table}[t]  
\centering 
\captionsetup{skip=2pt}
\caption{\textbf{Ablations.} Visual-quality and view-synchronization metrics on various factors.}\label{tab:ablation_components} 
\resizebox{0.7\textwidth}{!}{% 
\begin{tabular}{ll|ccc|c} 
\toprule 
\multicolumn{2}{c|}{Variant} & 
\multicolumn{3}{c|}{View Synchronization} & 
Temporal Consistency \\ 
\cmidrule(lr){3-5} 
\cmidrule(lr){6-6} & & Mat.\ Pix.\ $\uparrow$ & FVD-V $\downarrow$ & CLIP-V $\uparrow$ & CLIP-F $\uparrow$ \\ 
\midrule \multirow{2}{*}{Noise} 
& standard & 222 & 518.77 & 88.76 & 97.34 \\ 
& w/o whitening & 208 & 468.39 & 88.53 & 96.92 \\ 
& w/\hspace{0.5em} whitening & 267 & 461.44 & 91.05 & 97.61 \\ 
\midrule 
\multirow{4}{*}{Conditioning} & +none (I2V) & 224 & 830.00 & 87.53 & 97.80 \\ 
& +render view & 231 & 574.50 & 89.34 & 97.59 \\ 
& +source view & 319 & 530.42 & 91.41 & 97.95 \\ 
& +both & 267 & 461.44 & 91.05 & 97.61 \\ 
\midrule \multirow{4}{*}{Fine tuning} 
& LoRA 128 & 214 & 494.43 & 88.33 & 96.97 \\ 
& LoRA 256 & 267 & 461.44 & 91.05 & 97.61 \\ 
& LoRA 512 & 206 & 490.21 & 88.38 & 96.92 \\ 
\midrule \multirow{2}{*}{Backbone} 
& Wan2.1-1.3B & 217 & 689.30 & 88.23 & 97.25 \\ 
& Wan2.2-5B & 267 & 461.44 & 91.05 & 97.61 \\
& Wan2.1-14B & 308 & 442.20 & 90.88 & 97.49 \\ 
& Wan2.2-14B & 313 & 438.69 & 91.17 & 97.75 \\
\bottomrule 
\end{tabular}% 
} 
\end{table}
\vspace{-1em}

\section{Conclusion}
\label{sec:conclusion}
\vspace{-6pt}
We presented \name, a framework for synchronized novel-view video generation that coordinates diffusion
at the source of its stochasticity. By attaching persistent noise to a canonical 4D representation and
rasterizing it into target cameras, \name turns scene correspondence into a property of the diffusion
prior rather than a responsibility of image-plane warping or additional cross-view architecture. A
view-consistent whitening step keeps this structured prior compatible with the distribution expected by
the pretrained model, allowing the approach to be integrated through lightweight adaptation without
modifying the diffusion backbone.
These findings suggest that synchronized generation can benefit from treating noise as a world-level
representation: the same stochastic ingredients can follow content across viewpoints and time while
leaving the model free to synthesize newly revealed regions.

\clearpage
\bibliography{refs}

\appendix
\clearpage

\setcounter{page}{1}

\section{Supplementary}
\label{sec:supplementary}

\subsection{Automated Evaluation Set Curation for OpenSora}
\label{sec:eval_curation}

Reshooting a real video admits no ground-truth target view, so the only thing a
real-world benchmark can measure is behavior on the clips it happens to contain.
That makes the composition of the test set part of the result. A set assembled by
hand can flatter a method through nothing more than which scenes its author found
compelling, and it cannot be reproduced by anyone else. We therefore built
Multi-ViewBench, the 200-clip evaluation set used in Sec.~\ref{sec:results}, so
that every accept or reject decision is either a recorded model judgement or a
deterministic function of recorded judgements. No clip was ever chosen, dropped,
or reordered by inspection.

The pipeline has two stages. Stage~1 asks a vision-language model to judge each
candidate clip once and writes the judgement to disk. Stage~2 never calls a model
again: it filters, de-duplicates, and ranks using only the fields Stage~1 already
recorded, so the entire path from corpus to final CSV can be replayed offline from
the released judgements.

\subsubsection{Source corpus}
\label{sec:eval_corpus} 

Candidates come from the Mixkit subset of Open-Sora-Plan~v1.1.0~\cite{lin2024opensoraplan},
which contains 9,194 stock-footage \texttt{.mp4} clips filed under 40 topical
category folders (\texttt{forest}, \texttt{Street}, \texttt{Couple},
\texttt{Bicycle}, and so on). Two properties of this corpus shape the pipeline.
First, the category folders are editorial tags rather than a partition: the same
physical clip is frequently filed under several categories, so the corpus contains
substantially fewer distinct videos than files. Second, the clips are
heterogeneous in exactly the ways that matter for camera-controlled generation,
mixing locked-off tripod shots, time-lapses, portrait-orientation footage, and
heavy optical effects with the moving-camera material we actually want.

\subsubsection{Stage 1: structured vision-language judgement}
\label{sec:eval_stage1}

Each clip is judged independently by Gemini~3.1~Pro
(\texttt{gemini-3.1-pro-preview}), accessed through an OpenAI-compatible
inference gateway. We sample 12 uniformly spaced frames per clip, rescale each so
its longest side is at most 512 pixels, JPEG-encode them, and submit the sequence
with a fixed system prompt at temperature~0.2. The response is constrained to a
JSON schema, so the model cannot answer in prose or omit a field, and every reply
is machine-readable without parsing heuristics. One JSON file is written per clip,
which makes the pass resumable: an interrupted run skips clips whose output
already exists.

The schema, summarized in Table~\ref{tab:eval_schema}, has four parts. Six
boolean \emph{gates} encode hard admissibility requirements. A
\emph{composition bucket} assigns the clip to \texttt{simple} or \texttt{complex}
and is the field that later defines our two splits: \texttt{simple} denotes a
single subject under smooth tracking against a clean background, and
\texttt{complex} denotes multiple subjects, cluttered environments, distinct depth
layers, or dynamic backgrounds. Four \emph{semantic scores} rate the properties
that make a clip informative for view synthesis. Finally, three free-text
\emph{reasoning} fields record why the model answered as it did; these are
released, and Stage~2 consults them in exactly one narrow place
(Sec.~\ref{sec:eval_stage2}).

\begin{table}[t]
  \centering
  \caption{\textbf{Stage-1 judgement schema.} The model must return every field.
  Gates are hard requirements: a clip is admissible only if all six are true.
  The composition bucket defines the \emph{standard}/\emph{complex} split, and the
  semantic scores feed the richness index of Sec.~\ref{sec:eval_richness}.}
  \label{tab:eval_schema}
  \footnotesize
  \begin{tabular}{llp{0.44\textwidth}}
  \toprule
  Field & Range & Requirement or meaning \\
  \midrule
  \multicolumn{3}{l}{\emph{Gates} (all must be true)} \\
  \texttt{is\_single\_continuous\_shot} & bool & No cuts or transitions. \\
  \texttt{is\_real\_time\_playback}     & bool & No time-lapse, hyper-lapse, or slow motion. \\
  \texttt{is\_physically\_plausible}    & bool & No VFX blends, text overlays, or screen-in-screen. \\
  \texttt{is\_motion\_blur\_acceptable} & bool & No motion blur severe enough to destroy structure. \\
  \texttt{is\_focal\_blur\_acceptable}  & bool & Not excessively out of focus. \\
  \texttt{is\_depth\_of\_field\_acceptable} & bool & No extreme shallow depth of field. \\
  \midrule
  \multicolumn{3}{l}{\emph{Composition}} \\
  \texttt{composition\_bucket} & 2 labels & \texttt{simple} for a single tracked subject on a clean background, \texttt{complex} for multiple subjects, clutter, or depth layers. \\
  \midrule
  \multicolumn{3}{l}{\emph{Semantic scores}} \\
  \texttt{geometric\_parallax} & 1--5 & 1 is flat, 5 is deep multi-layered depth. \\
  \texttt{occlusion\_potential} & 1--5 & 1 is no meaningful occlusion, 5 is heavy self-occlusion or disocclusion. \\
  \texttt{camera\_motion} & 5 labels & \texttt{static}, \texttt{dynamic}, or one of three optical-effect labels. \\
  \texttt{non\_rigid\_motion} & 4 labels & \texttt{none}, \texttt{low}, \texttt{mid}, or \texttt{high} object articulation. \\
  \midrule
  \multicolumn{3}{l}{\emph{Reasoning} (released, used only as described in Sec.~\ref{sec:eval_stage2})} \\
  \texttt{gates\_reasoning} & text & Justification for the six gate verdicts. \\
  \texttt{composition\_reasoning} & text & Justification for the composition bucket. \\
  \texttt{scores\_reasoning} & text & Justification for the four semantic scores. \\
  \bottomrule
  \end{tabular}
\end{table}

\subsubsection{The richness index}
\label{sec:eval_richness}

Gates decide admissibility, but among admissible clips we still need an ordering,
and we want it to prefer clips that actually exercise synchronized generation. A
clip is informative to the degree that it has depth to be parallaxed, geometry
that occludes and disoccludes as the viewpoint moves, a camera that moves at all,
and subjects that deform rather than translate rigidly. We combine the four
semantic scores into a single \emph{richness index}
\begin{equation}
  \label{eq:richness}
  \rho \;=\; \underbrace{g}_{\text{parallax}} \cdot \underbrace{o}_{\text{occlusion}}
       \cdot \; w_{\mathrm{cam}} \cdot w_{\mathrm{nr}},
\end{equation}
where $g, o \in \{1,\dots,5\}$ are \texttt{geometric\_parallax} and
\texttt{occlusion\_potential}, and the two weights are read off the categorical
labels: $w_{\mathrm{cam}} = 2.0$ for \texttt{dynamic}, $1.2$ for any optical-effect
label, and $1.0$ for \texttt{static}; $w_{\mathrm{nr}} = 1.0, 1.2, 1.5, 2.0$ for
\texttt{none}, \texttt{low}, \texttt{mid}, \texttt{high} respectively. Richness
ranges over $[1, 100]$.

The product form is deliberate. Depth without occlusion, or occlusion without
camera motion, describes a clip that is easy for a single-view model and therefore
uninformative here, and a product refuses to compensate for a low factor with a
high one. Because Stage~2 admits only moving-camera clips,
$w_{\mathrm{cam}} = 2.0$ holds throughout the final pool and richness is
effectively $2 g o \, w_{\mathrm{nr}}$ there.

\subsubsection{Stage 2: deterministic filtering}
\label{sec:eval_stage2}

Stage~2 reduces the scored corpus to a candidate pool through five filters applied
in a fixed order. Table~\ref{tab:eval_funnel} reports the count surviving each.

\textbf{Gate pass.} A clip is kept only if all six gates are true.

\textbf{Landscape.} We probe each survivor with a video decoder and keep clips
whose frame width exceeds their height, discarding portrait and square footage
that no baseline is trained to reshoot. This is the only filter that reads pixels;
frame dimensions are cached so reruns do not repeat it.

\textbf{De-duplication.} Because one clip may be filed under several categories,
Stage~1 scored each copy independently, and copies routinely disagree with one
another. We collapse to one row per unique file name, keeping the highest-richness
copy and breaking ties deterministically on (category, judgement filename) so the
outcome does not depend on directory iteration order. The discarded copies'
categories are preserved in an \texttt{alias\_categories} field rather than
thrown away.

\textbf{Static-source removal.} A clip whose camera body never moves gives a
camera-controlled model no viewpoint baseline to reshoot from, so we keep only
physically moving cameras. The \texttt{camera\_motion} label carries this
directly for \texttt{static} and \texttt{dynamic}. The optical-effect labels are
genuinely ambiguous, since a rack focus or lens flare is usually captured on a
locked-off tripod but occasionally accompanies a real dolly move. This is the one
place where the reasoning text is consulted: a regular expression looks for
explicit camera-body motion (pan, tilt, track, dolly, handheld) to adjudicate the
optical-effect clips, and to demote the rare \texttt{dynamic} clip whose own
justification states the camera is stationary. The text is never used to overrule
an unambiguous label.

\textbf{Optical-effect removal.} Any clip still carrying a
\texttt{focal\_effects}, \texttt{aperture\_effects}, or \texttt{sensor\_effects}
label is dropped outright, since lens-driven apparent motion is not viewpoint
change.

\begin{table}[t]
  \centering
  \caption{\textbf{Curation funnel.} Counts at the time the benchmark was frozen,
  when 7,561 of the 9,194 corpus clips had been scored. De-duplication runs after
  the gate and landscape filters but before static-source classification; see
  Sec.~\ref{sec:eval_ordering} for why the order matters.}
  \label{tab:eval_funnel}
  \small
  \begin{tabular}{llrl}
  \toprule
  \# & Stage & Surviving & Removed \\
  \midrule
    & Stage-1 judgements available            & 7,561 & \\
  1 & Gate pass (all six gates true)          & 6,617 & 944 gate failures \\
  2 & Landscape (width $>$ height)            & 5,973 & 644 portrait or square \\
  3 & De-duplication (unique clips)           & 4,541 & 1,432 copies of 1,197 clips \\
  4 & Moving camera body                      & 2,231 & 2,310 static-source \\
  5 & Optical-effect removal                  & \textbf{2,227} & 4 optical-effect \\
  \midrule
    & \quad of which \texttt{complex}         & 1,823 & \\
    & \quad of which \texttt{simple}          &   404 & \\
  \bottomrule
  \end{tabular}
\end{table}

\subsubsection{Why the filter order matters}
\label{sec:eval_ordering}

Two orderings in Table~\ref{tab:eval_funnel} are load-bearing rather than
incidental.

De-duplication runs \emph{after} the gates so that a copy which fails a gate can
never win the richness comparison and readmit a clip that should have been
rejected. It runs \emph{before} static-source classification so that each physical
clip receives exactly one verdict. Were the order reversed, a clip could enter the
pool through whichever of its category copies happened to be scored most
favorably, which would let editorial filing decisions leak into the benchmark.

Static-source removal precedes optical-effect removal because the two overlap: the
adjudication step already reclassifies most optical-effect clips as static, so only
4 clips survive to be dropped by the explicit optical filter. Reversing the two
would discard the handful of optical-effect clips that do accompany real camera
motion.

\subsubsection{Constructing the two splits}
\label{sec:eval_splits}

The pool is strongly unbalanced, holding 1,823 complex clips against 404 simple
ones, because static-source removal strips proportionally more single-subject
tracking shots. The two splits are therefore built differently, and deliberately
so.

\textbf{Standard (100 clips).} Simple compositions are the comparatively solved
case, so this split is a compact reference rather than half the benchmark, and its
specification is exactly \emph{take the 100 highest-richness simple clips}. No
category cap and no diversity heuristic apply. With only 404 candidates, imposing
balance here would trade away richness for very little diversity.

\textbf{Complex (100 clips).} A pure richness sort over 1,823 complex clips
collapses onto whichever categories are largest, and the supply is severely skewed
(150 clips for \texttt{forest} against 1 for \texttt{Cats}). We therefore
construct this split in two steps. First we select a category-balanced set of
1,000 clips in which balance is a hard constraint and richness is maximized only
within it. Because an equal quota of $1000/40 = 25$ per category is unavailable
from most categories, we allocate quotas by max-min fair water-filling: give every
category the same quota, let categories that cannot meet it contribute everything
they have, and redistribute the shortfall equally among categories that still have
supply, repeating until the budget is spent. This maximizes the smallest
category's share, then the next smallest, and degenerates to an equal split when
supply permits. Each category then contributes its own richest clips up to its
quota, so balance is never traded for richness. The resulting 1,000 clips span all
40 categories with between 1 and 38 clips each, a mean of 25.0, and a normalized
category entropy of 0.940 against a maximum of 1. Second, we take the 100
highest-richness clips from this balanced set to form the released split.

The two-step construction is worth stating plainly, because balance and richness
pull against each other and the second step partially undoes the first. Balancing
at 1,000 determines \emph{which} clips are available to be ranked, preventing the
fat categories from monopolizing the candidate set; truncating to the richest 100
then concentrates on 19 of the 40 categories. What balance buys is that the final
100 is drawn from a category-fair pool rather than from a raw richness sort over
the whole corpus, not that the final 100 is itself uniform over categories.

Both selections sort by richness descending with ties broken deterministically on
(category, file name), and in both cases the cut falls inside a tie group: the
standard split cuts within the $\rho = 24$ tier, and the complex split takes 39
clips above $\rho = 64$ and then 61 of the 210 clips at exactly $\rho = 64$. The
tie-break rule is what makes these selections reproducible, and it is the reason
the released CSVs are byte-identical across reruns.

\subsubsection{Resulting statistics}
\label{sec:eval_stats}

Table~\ref{tab:eval_splits} characterizes the two released splits. They are
disjoint by construction, sharing no clip, and together form the 200 clips of
Multi-ViewBench. The richness distributions are almost separated: every complex
clip scores at least 64, a value that exactly one standard clip attains and none
exceeds, against a standard mean of 30.9. The gap is visible in the
underlying attributes rather than being an artifact of the index, since every
complex clip carries high non-rigid motion and a parallax rating of at least 4,
whereas the standard split is concentrated at parallax 2 to 3.

All 200 clips are natively $1920 \times 1080$, and are standardized to 5-second
clips at 15\,fps and 480p for evaluation as described in
Sec.~\ref{sec:datasets}.

\begin{table}[t]
  \centering
  \caption{\textbf{The two released splits of Multi-ViewBench.} Both hold 100
  clips, are disjoint, and are drawn from the 2,227-clip pool of
  Table~\ref{tab:eval_funnel}. Richness $\rho$ is defined in
  Eq.~\ref{eq:richness}. Attribute rows give the distribution over the 100 clips.}
  \label{tab:eval_splits}
  \small
  \begin{tabular}{lcc}
  \toprule
   & Standard & Complex \\
  \midrule
  Composition bucket            & \texttt{simple} & \texttt{complex} \\
  Candidates in pool            & 404 & 1,823 \\
  Selection                     & top 100 by $\rho$ & balance to 1,000, then top 100 by $\rho$ \\
  \midrule
  $\rho$ range                  & 24.0--64.0 & 64.0--100.0 \\
  $\rho$ mean                   & 30.9 & 71.2 \\
  $\rho$ median                 & 32.0 & 64.0 \\
  \midrule
  Categories represented        & 19 & 19 \\
  Largest category              & \texttt{Music} (24) & \texttt{Couple} (26) \\
  Clips with alias categories   & 20 & 47 \\
  \midrule
  Geometric parallax            & 2:\,57 \; 3:\,36 \; 4:\,7 & 4:\,83 \; 5:\,17 \\
  Occlusion potential           & 2:\,9 \; 3:\,63 \; 4:\,28 & 4:\,73 \; 5:\,27 \\
  Non-rigid motion              & low:\,1 \; mid:\,4 \; high:\,95 & high:\,100 \\
  Camera motion                 & \texttt{dynamic}:\,100 & \texttt{dynamic}:\,100 \\
  \bottomrule
  \end{tabular}
\end{table}

\subsubsection{Reproducibility and release}
\label{sec:eval_release}

We release the complete curation trail, not only the final clip list, so that the
benchmark can be audited and extended rather than merely consumed:

\begin{itemize}
  \item the Stage-1 and Stage-2 code, including the verbatim system prompt, the
        response schema, and the richness weights;
  \item the per-clip Stage-1 judgement JSONs, covering all gates, scores, and
        reasoning texts, for every clip scored, including those the pipeline
        rejected;
  \item the intermediate pool CSV and both released split CSVs, each carrying all
        18 columns, so a reader can recompute any count in
        Table~\ref{tab:eval_funnel} or reproduce either selection; and
  \item the exact clip identities and source paths, together with a SHA-256
        manifest over the frozen CSVs.
\end{itemize}

The frozen CSVs are the integrity anchor for every number reported in
Sec.~\ref{sec:results}: the splits were fixed before any baseline was evaluated
and were not revised afterwards. Stage~2 is deterministic and seed-free, so given
the released judgements the splits are recoverable exactly.

\subsection{Novel-Trajectory Synthesis} \label{sec:opensora_trajectories}

\definecolor{gridlight}{RGB}{207,211,216}
\definecolor{griddark}{RGB}{174,180,187}
\definecolor{elevateblue}{RGB}{24,164,224}
\definecolor{arcviolet}{RGB}{177,34,185}
\definecolor{zoomorange}{RGB}{255,132,52}
\definecolor{camerablue}{RGB}{43,111,218}

\begin{figure}[h]
\centering
\begin{minipage}[c]{0.55\textwidth}
\caption{\textbf{Schematic of the six synthesized evaluation trajectories}; Every trajectory is a camera-local SE(3) delta $\Delta_t$ composed \emph{after} the source camera's pose at each
frame (Eq.~\ref{eq:trajectory_delta}), so all six depart from the same source pose (\texttt{Start}) at $t=0$ and diverge
from it over the clip. Arc left/right (violet) sweep the camera
around a pivot ahead of it in the local horizontal (right--forward) plane; Dolly in/out (orange)
dolly along the local forward axis toward/away from that pivot; Pedestal-up-Tilt-down and Pedestal-down-Tilt-up (blue)
shift the camera along its local down axis while tilting about its local right axis so the same pivot stays centred in
frame.}
\label{fig:traj_schematic}
\end{minipage}%
\hspace{0.05\textwidth}%
\begin{minipage}[c]{0.35\textwidth}
\centering
\resizebox{\textwidth}{!}{%
\begin{tikzpicture}[
    x=1cm,
    y=1cm,
    line cap=round,
    line join=round,
    font=\small,
    >={Latex[length=3.2mm,width=2.4mm]}
]

\def\FrameWidth{10.6}
\def\FrameHeight{10.6}
\def\CornerRadius{0.85cm}

\coordinate (C) at (5.20,4.35);

\coordinate (BLT) at (1.90,8.70);
\coordinate (BRT) at (8.65,8.70);
\coordinate (BRB) at (8.65,1.85);
\coordinate (BLB) at (1.90,1.85);

\coordinate (FLT) at (0.35,10.10);
\coordinate (FRT) at (10.25,10.10);
\coordinate (FRB) at (10.25,0.45);
\coordinate (FLB) at (0.35,0.45);

\fill[black,opacity=0.14,rounded corners=\CornerRadius]
  (0.18,-0.18) rectangle ++(\FrameWidth,\FrameHeight);

\fill[white,rounded corners=\CornerRadius]
  (0,0) rectangle ++(\FrameWidth,\FrameHeight);

\begin{scope}
\clip[rounded corners=\CornerRadius]
  (0,0) rectangle ++(\FrameWidth,\FrameHeight);

\fill[white] (0,0) rectangle ++(\FrameWidth,\FrameHeight);

\draw[griddark,line width=0.8pt]
  (BLT) -- (BRT) -- (BRB) -- (BLB) -- cycle;

\draw[griddark,line width=0.8pt] (FLT) -- (BLT);
\draw[griddark,line width=0.8pt] (FRT) -- (BRT);
\draw[griddark,line width=0.8pt] (FRB) -- (BRB);
\draw[griddark,line width=0.8pt] (FLB) -- (BLB);

\foreach \x in {2.90,3.90,4.90,5.90,6.90,7.90}
  \draw[gridlight,line width=0.42pt] (\x,1.85) -- (\x,8.70);

\foreach \y in {2.85,3.85,4.85,5.85,6.85,7.85}
  \draw[gridlight,line width=0.42pt] (1.90,\y) -- (8.65,\y);

\foreach \xfront/\xback in {
    1.80/2.90,
    3.35/3.90,
    4.55/4.90,
    5.75/5.90,
    6.95/6.90,
    8.10/7.90,
    9.25/8.65
}
  \draw[gridlight,line width=0.42pt]
    (\xfront,10.10) -- (\xback,8.70);

\foreach \t in {0.18,0.36,0.54,0.72}
  \draw[gridlight,line width=0.42pt]
    ($(FLT)!\t!(BLT)$) -- ($(FRT)!\t!(BRT)$);

\foreach \xfront/\xback in {
    1.00/1.90,
    2.35/2.90,
    3.65/3.90,
    4.75/4.90,
    5.80/5.90,
    6.90/6.90,
    8.05/7.90,
    9.35/8.65
}
  \draw[gridlight,line width=0.42pt]
    (\xfront,0.45) -- (\xback,1.85);

\foreach \t in {0.18,0.36,0.54,0.72}
  \draw[gridlight,line width=0.42pt]
    ($(FLB)!\t!(BLB)$) -- ($(FRB)!\t!(BRB)$);

\foreach \t in {0.16,0.32,0.48,0.64,0.80}
  \draw[gridlight,line width=0.42pt]
    ($(FLT)!\t!(FLB)$) -- ($(BLT)!\t!(BLB)$);

\foreach \t in {0.25,0.50,0.75}
  \draw[gridlight,line width=0.42pt]
    ($(FLT)!\t!(BLT)$) -- ($(FLB)!\t!(BLB)$);

\foreach \t in {0.16,0.32,0.48,0.64,0.80}
  \draw[gridlight,line width=0.42pt]
    ($(FRT)!\t!(FRB)$) -- ($(BRT)!\t!(BRB)$);

\foreach \t in {0.25,0.50,0.75}
  \draw[gridlight,line width=0.42pt]
    ($(FRT)!\t!(BRT)$) -- ($(FRB)!\t!(BRB)$);

\draw[
    elevateblue,
    line width=2.35pt,
    -{Latex[length=4.1mm,width=3.1mm]}
]
($(C)+(0.00,0.35)$)
.. controls ($(C)+(-0.40,1.40)$) and ($(C)+(-0.30,2.75)$) ..
($(C)+(0.28,4.05)$);

\draw[
    elevateblue,
    line width=2.35pt,
    -{Latex[length=4.1mm,width=3.1mm]}
]
($(C)+(0.03,-0.35)$)
.. controls ($(C)+(0.30,-1.25)$) and ($(C)+(0.34,-2.10)$) ..
($(C)+(-0.12,-2.75)$);

\draw[
    arcviolet,
    line width=2.35pt,
    -{Latex[length=4.1mm,width=3.1mm]}
]
($(C)+(-0.34,0.00)$)
.. controls ($(C)+(-1.65,0.10)$) and ($(C)+(-2.45,0.75)$) ..
($(C)+(-3.05,1.48)$);

\draw[
    arcviolet,
    line width=2.35pt,
    -{Latex[length=4.1mm,width=3.1mm]}
]
($(C)+(0.34,-0.03)$)
.. controls ($(C)+(1.62,-0.22)$) and ($(C)+(2.55,0.05)$) ..
($(C)+(3.18,0.85)$);

\draw[
    zoomorange,
    line width=2.25pt,
    -{Latex[length=4.0mm,width=3.0mm]}
]
($(C)+(0.30,0.25)$) -- ($(C)+(2.13,2.15)$);

\draw[
    zoomorange,
    line width=2.25pt,
    -{Latex[length=4.0mm,width=3.0mm]}
]
($(C)+(-0.25,-0.28)$) -- ($(C)+(-2.20,-2.20)$);

\fill[white] (C) circle[radius=0.72cm];
\fill[camerablue] (C) circle[radius=0.66cm];
\draw[camerablue!75!black,line width=0.45pt]
  (C) circle[radius=0.68cm];

\draw[white,line width=1.65pt,rounded corners=0.07cm]
  ($(C)+(-0.35,-0.22)$) rectangle ($(C)+(0.35,0.22)$);

\draw[white,line width=1.65pt,rounded corners=0.035cm]
  ($(C)+(-0.18,0.22)$) --
  ($(C)+(-0.11,0.34)$) --
  ($(C)+(0.12,0.34)$) --
  ($(C)+(0.18,0.22)$);

\draw[white,line width=1.65pt]
  (C) circle[radius=0.18cm];

\fill[white]
  ($(C)+(-0.24,0.08)$) circle[radius=0.035cm];

\node[anchor=east,align=right,inner sep=1pt,font=\fontsize{16}{18}\selectfont]
  at ($(C)+(-0.47,2.55)$)
  {Pedestal $\uparrow$\\[-1pt]Tilt $\downarrow$};

\node[anchor=west,align=left,inner sep=1pt,font=\fontsize{16}{18}\selectfont]
  at ($(C)+(0.38,-1.52)$)
  {Pedestal $\downarrow$\\[-1pt]Tilt $\uparrow$};

\node[anchor=east,inner sep=1pt,font=\fontsize{16}{18}\selectfont]
  at ($(C)+(-3.35,0.43)$) {Arc left};

\node[anchor=west,inner sep=1pt,font=\fontsize{16}{18}\selectfont]
  at ($(C)+(3.06,-0.14)$) {Arc right};

\node[anchor=south west,inner sep=1pt,font=\fontsize{16}{18}\selectfont]
  at ($(C)+(1.47,1.84)$) {Dolly-out};

\node[anchor=north east,inner sep=1pt,font=\fontsize{16}{18}\selectfont]
  at ($(C)+(-1.42,-1.43)$) {Dolly-in};

\node[anchor=west,inner sep=1pt,font=\fontsize{16}{18}\selectfont]
  at ($(C)+(0.77,0.25)$) {Start};

\end{scope}

\draw[black,line width=1.25pt,rounded corners=\CornerRadius]
  (0,0) rectangle ++(\FrameWidth,\FrameHeight);

\end{tikzpicture}%
}
\end{minipage}
\end{figure}

Because each OpenSora scene provides a single source camera, there is no held-out ground-truth viewpoint against which to
evaluate reshoots on this data; we instead procedurally synthesize novel target trajectories $\boldsymbol{\Pi}^{\mathrm{tar}}_{1:T}$
that follow the source camera's own motion rather than being anchored to its initial pose.
Concretely, at every frame $t$ we
compose a time-varying SE(3) offset $\Delta_t$ \emph{after} the source pose, expressed in the source camera's own local axes
(OpenCV convention):
\begin{equation}
    \boldsymbol{\Pi}^{\mathrm{tar}}_t = \boldsymbol{\Pi}^{\mathrm{src}}_t \, \Delta_t .
    \label{eq:trajectory_delta}
\end{equation}
Since OpenSora sources are handheld or drone footage whose camera keeps moving throughout the clip, anchoring $\Delta_t$ to a
fixed world pose would let the novel and source views drift apart in content over the clip; composing it on the current
source pose instead keeps the novel view co-moving with the source, offset from it by a structured, time-varying amount.
Every $\Delta_t$ is the identity at $t=0$, so each synthesized trajectory has aligned initial pose with the source 
trajectory and diverges from it over the clip.
We synthesize six trajectory families this way, illustrated schematically in Fig.~\ref{fig:traj_schematic}:

\textbf{Arc~left/right} sweep $\Delta_t$'s rotation up to $60^{\circ}$ about the source camera's local down axis,
pivoting around a point a distance $d$ ahead of the camera.

\textbf{Dolly~in/out} shift $\Delta_t$ purely along the local forward axis, so the same pivot ends up three times 
closer or three times farther by the last frame.

\textbf{Pedestal-up-Tilt-down} and \textbf{Pedestal-down-Tilt-up} translate $\Delta_t$ along the local down axis by up to
30\% of the estimated distance to the surface being approached (e.g.\ the floor), while simultaneously tilting about the
local right axis so the same pivot stays centred in frame throughout. 

The pivot distance $d$ and the vertical hit-surface
distance used to size the pedestal trajectories are both estimated per scene from the reconstructed Gaussians $\mathcal{G}$
(Sec.~\ref{sec:prelim_dynamic_gs}). For arc/dolly, $d$ is the depth of the nearest surface in a small patch at the centre of
frame~0; for pedestal trajectories, the hit-surface distance is the mean depth over the bottom band of frame~0. Both fall back to the
median distance of the scene (based on $\{\bar{\boldsymbol{\mu}}_i\}$) from the source camera when a per-pixel depth estimate 
is unavailable for example a video from a landscape, so trajectory magnitudes remain sensible regardless of the arbitrary 
metric scale a feed-forward backbone assigns to a given scene.

\subsection{Local Decorrelation of the Whitening Projection} \label{sec:whitening_proof}

Sec.~\ref{sec:view_consistent_whitening} argues informally that the random-Fourier-feature (RFF) projection field
$\mathbf{Q}_p$ (Eq.~\plaineqref{eq:projection_qr}) suppresses the cross-pixel term $\mathbf{Q}_p^\top\mathbf{Q}_q$ in the projected
covariance (Eq.~\plaineqref{eq:projected_cross_covariance}) whenever two pixels' canonical anchors $\mathbf{X}_p,\mathbf{X}_q$ are
distinct. We make this precise with a short argument: two \emph{distinct} 3D anchors are mapped, through the many-harmonic
the fixed sinusoidal map of Eq.~\plaineqref{eq:projection_preconditioning} produces
essentially unrelated pre-projection matrices $\mathbf A_p,\mathbf A_q$; QR
factorizing two unrelated matrices yields two unrelated orthonormal bases $\mathbf Q_p,\mathbf Q_q$ of $\mathbb R^H$; and two
unrelated $C$-dimensional bases of a much larger $H$-dimensional space are nearly orthogonal to each other, i.e.\
$\mathbf Q_p^\top\mathbf Q_q\approx\mathbf 0$. This is what should suppress the correlation $\rho_{pq}$ in
Eq.~\plaineqref{eq:projected_cross_covariance}. Note that the target here is the
\emph{zero} matrix, not the identity: $\mathbf Q_p^\top\mathbf Q_p=\mathbf I_C$
holds \emph{exactly} by the reduced QR factorization in
Eq.~\plaineqref{eq:projection_qr}; it is only the $p\ne q$ case that needs an
argument. Together, an exact identity on
the diagonal and a near-zero off-diagonal term is what makes the noise field as a whole resemble an i.i.d.\ (identity-covariance)
Gaussian field, which is the property the pretrained diffusion model expects.

\begin{assumption}[Distinct anchors encode as generic bases] \label{ass:generic}
When $\mathbf X_p\ne\mathbf X_q$, the direct Fourier map of Eq.~\plaineqref{eq:projection_preconditioning} uses $H$
frequency vectors $\boldsymbol\omega_1,\dots,\boldsymbol\omega_H$ spanning many distinct, incommensurate harmonics (cf.\ random
Fourier features~\cite{rahimi2007random}), so that the resulting $\mathbf A_p,\mathbf A_q\in\mathbb
R^{H\times C}$ -- carry no shared structure beyond what is forced by $\mathbf X_p$ being close to $\mathbf X_q$. We formalize
``no shared structure'' by treating the orthonormal bases $\mathbf Q_p,\mathbf Q_q$ obtained by QR-factorizing $\mathbf A_p,\mathbf
A_q$ (Eq.~\plaineqref{eq:projection_qr}) as two independent, uniformly random $C$-dimensional subspaces of $\mathbb R^H$ (i.e.\
Haar-random on the Grassmannian) whenever $\mathbf X_p\ne\mathbf X_q$. This is the sense in which ``many harmonics'' matters:
it is what keeps two different points from being accidentally mapped to correlated bases.
\end{assumption}

\begin{lemma}[Random bases in high dimension are nearly orthogonal] \label{lem:random_subspace}
Let $\mathbf U,\mathbf V\in\mathbb R^{H\times C}$ have orthonormal columns, with $\mathbf V$ Haar-random on the Grassmannian and
independent of the fixed $\mathbf U$. Then
\begin{equation}
\mathbb E\big[\mathbf U^\top\mathbf V\big] = \mathbf 0, \qquad \mathbb E\big[\|\mathbf U^\top\mathbf V\|_F^2\big] =
\frac{C^2}{H},
\label{eq:random_subspace_overlap}
\end{equation}
so, by Jensen's inequality, $\mathbb E\big[\|\mathbf U^\top\mathbf V\|_F\big] \le C/\sqrt H$, or equivalently
$\mathbb E\big[\tfrac{1}{\sqrt C}\|\mathbf U^\top\mathbf V\|_F\big] \le \sqrt{C/H} \to 0$ as $H\to\infty$ for fixed $C$.
\end{lemma}
\begin{proof}
Let $\mathbf u_i,\mathbf v_j$ denote the unit-norm columns of $\mathbf U,\mathbf V$. Since $\mathbf v_j$ is a uniformly random
unit vector in $\mathbb R^H$ independent of the fixed $\mathbf u_i$, rotational symmetry gives $\mathbb E[\mathbf
u_i^\top\mathbf v_j]=0$ and, since $\sum_{h=1}^H \mathbb E[v_{j,h}^2]=1$ with every coordinate direction equivalent by symmetry,
$\mathbb E\big[(\mathbf u_i^\top\mathbf v_j)^2\big]=\tfrac1H$ for every unit vector $\mathbf u_i$. Summing over the $C^2$
entries of $\mathbf U^\top\mathbf V$ gives Eq.~\plaineqref{eq:random_subspace_overlap}. This is exact: each column $\mathbf v_j$
is marginally uniform on the unit sphere, so the orthogonality among the columns of $\mathbf V$ does not affect the
second moment. The bound on $\mathbb E\|\mathbf U^\top\mathbf V\|_F$ then follows from
$\mathbb E[X]\le\sqrt{\mathbb E[X^2]}$.
\end{proof}

\begin{proposition}[Local decorrelation of the whitening projection] \label{prop:decorrelation}
Under Assumption~\ref{ass:generic}, for two pixels $p\ne q$ with distinct anchors $\mathbf X_p\ne\mathbf X_q$,
\begin{equation}
\mathbb E\big[\mathbf Q_p^\top\mathbf Q_q\big] = \mathbf 0, \qquad
\frac{1}{\sqrt C}\,\mathbb E\big[\|\mathbf Q_p^\top\mathbf Q_q\|_F\big] \le \sqrt{\frac{C}{H}},
\label{eq:prop_result}
\end{equation}
while $\mathbf Q_p^\top\mathbf Q_p=\mathbf I_C$ exactly, for which the same normalized norm equals $1$. Consequently, in the
projected cross-covariance of Eq.~\plaineqref{eq:projected_cross_covariance},
\begin{equation}
\frac{1}{\sqrt C}\,\mathbb E\big[\|\operatorname{Cov}(\boldsymbol\epsilon_p^{\mathrm{3D}},\boldsymbol\epsilon_q^{\mathrm{3D}})\|_F\big] =
\frac{|\rho_{pq}|}{\sqrt C}\,\mathbb E\big[\|\mathbf Q_p^\top\mathbf Q_q\|_F\big] \le |\rho_{pq}|\sqrt{\frac{C}{H}}
\xrightarrow[H\to\infty]{} 0,
\label{eq:prop_bound}
\end{equation}
so the rasterization-induced correlation $\rho_{pq}$ of Eq.~\plaineqref{eq:normalized_cross_covariance} is suppressed for any two
pixels whose canonical anchors differ, and increasingly so as the noise dimension $H$ grows relative to $C$.
\end{proposition}
\begin{proof}
Immediate from Lemma~\ref{lem:random_subspace} applied under Assumption~\ref{ass:generic} (with $\mathbf U=\mathbf Q_p,\mathbf
V=\mathbf Q_q$), together with $\mathbf Q_p^\top\mathbf Q_p=\mathbf I_C$ from the reduced QR factorization
(Eq.~\plaineqref{eq:projection_qr}), and the definition $\operatorname{Cov}(\boldsymbol\epsilon_p^{\mathrm{3D}},
\boldsymbol\epsilon_q^{\mathrm{3D}}) = \rho_{pq}\mathbf Q_p^\top\mathbf Q_q$ of Eq.~\plaineqref{eq:projected_cross_covariance}.
\end{proof}

\begin{remark}
This argument explains \emph{why} distinctness of the anchor helps, but not \emph{how far apart} two anchors must be before
Assumption~\ref{ass:generic} is a good approximation: two points much closer together than the shortest wavelength present in
the fixed frequency vectors $\boldsymbol\omega_a$ still produce almost the same $\mathbf A_p\approx\mathbf A_q$ and
$\mathbf Q_p\approx\mathbf Q_q$, so $\mathbf Q_p^\top\mathbf Q_q\approx\mathbf I_C$ instead of $\mathbf 0$ -- this is exactly why
Sec.~\ref{sec:view_consistent_whitening} also needs \emph{high-frequency} harmonics, not merely \emph{many} of them: frequency
sets the length scale below which Assumption~\ref{ass:generic} stops holding and decorrelation degrades toward the pre-projection
correlation $\rho_{pq}$. This is also why the occlusion-boundary case flagged in Sec.~\ref{sec:view_consistent_whitening} is a
genuine exception: the blended anchor of Eq.~\plaineqref{eq:canonical_position} can there be arbitrarily close to both surfaces'
anchors, defeating the distinctness the argument relies on.
\end{remark}

\paragraph{Scope.} This proposition addresses only the decorrelation induced by the
canonical-position-dependent projection field. It does not establish exact whitening
of the complete rasterized noise field, since shared primitives, occlusion-boundary
blending, and visibility changes can preserve residual correlations. It also does not
model correspondence errors in unobserved regions caused by the finite background
proxy. Finally, Assumption~\ref{ass:generic} is an idealization: it treats the
projection bases induced by distinct canonical positions as independent Haar-random
subspaces, whereas in the actual method they are deterministic functions of the fixed
frequencies $\boldsymbol{\omega}_a$, phases $\phi_{a,c}$, and scale $s$. The quality of
this approximation therefore depends on the separation of canonical positions
relative to the frequencies used by the projection field. Empirically, we found that
H=128 works fine and we keep it constant through all the experiments.

\subsection{Noise Rasterization and View Consistency} \label{sec:rotation_equivariance}

Sec.~\ref{sec:canonical_3d_noise} claims that two cameras observing the same 3D content receive the same rasterized noise
because both rasterize the same canonical anchor. We prove this exactly, for \emph{arbitrary} pairs of camera poses -- not
only the pure-rotation special case -- and for the general 4D (dynamic) setting, where a deformation module has already
produced the time-dependent Gaussian parameters. In 4D Gaussian Splatting this result applies independently at each
timestamp, after the deformation module has produced the time-dependent parameters of the 3D Gaussian primitives; we
therefore fix a single timestamp $t$ throughout and do not address temporal consistency here (that is a separate claim,
discussed qualitatively in Sec.~\ref{sec:view_consistent_whitening}).

\paragraph{Setup at a fixed time.} Consider two time-dependent Gaussian centers
\begin{equation}
\mathbf p(t),\mathbf q(t)\in\mathbb R^3.
\end{equation}
These centers may be generated by a 4D deformation field $D$ from their canonical positions $\mathbf p_0,\mathbf q_0$ --
i.e.\ $\mathbf p_0,\mathbf q_0$ play the role of the canonical anchors $\bar{\boldsymbol\mu}_i$ of
Eq.~\plaineqref{eq:primitive_attributes}:
\begin{equation}
\mathbf p(t) = D(\mathbf p_0,t), \qquad \mathbf q(t) = D(\mathbf q_0,t).
\label{eq:deformation_field}
\end{equation}
Let the camera-to-world pose at timestamp $t$ be
\begin{equation}
T_{wc}(t) = \begin{bmatrix} R_{wc}(t) & \mathbf c(t) \\ \mathbf 0^\top & 1 \end{bmatrix} \in SE(3),
\end{equation}
where $R_{wc}(t)\in SO(3)$ describes the camera orientation and $\mathbf c(t)\in\mathbb R^3$ denotes the camera center in
world coordinates. The corresponding world-to-camera transformation is
\begin{equation}
T_{cw}(t) = T_{wc}(t)^{-1} = \begin{bmatrix} R_{cw}(t) & \mathbf t_{cw}(t) \\ \mathbf 0^\top & 1 \end{bmatrix}, \qquad
R_{cw}(t)=R_{wc}(t)^\top, \quad \mathbf t_{cw}(t) = -R_{cw}(t)\mathbf c(t).
\end{equation}
The camera-space coordinate of $\mathbf p(t)$ is therefore
\begin{equation}
\mathbf x_p^c(t) = R_{cw}(t)\mathbf p(t) + \mathbf t_{cw}(t) = R_{cw}(t)\big(\mathbf p(t)-\mathbf c(t)\big),
\end{equation}
and, similarly, $\mathbf x_q^c(t) = R_{cw}(t)\big(\mathbf q(t)-\mathbf c(t)\big)$. Let $\mathbf K$ denote the camera intrinsic
matrix and $\pi_K$ the perspective projection; for a camera-space point $\mathbf x=[x,y,z]^\top$,
\begin{equation}
\pi_K(\mathbf x) = \begin{bmatrix} f_x\frac{x}{z}+c_x \\[4pt] f_y\frac{y}{z}+c_y \end{bmatrix}.
\label{eq:perspective_projection}
\end{equation}
The projected pixel locations of $\mathbf p(t)$ and $\mathbf q(t)$ are thus
\begin{equation}
\mathbf u_p(t) = \pi_K\big(R_{cw}(t)(\mathbf p(t)-\mathbf c(t))\big), \qquad
\mathbf u_q(t) = \pi_K\big(R_{cw}(t)(\mathbf q(t)-\mathbf c(t))\big).
\label{eq:projected_pixels}
\end{equation}
Both are governed by the same world-to-camera transformation $T_{cw}(t)$, although their pixel locations generally differ,
since perspective projection is depth- and position-dependent.

\paragraph{Two cameras observing the same primitive.} Let $\Pi_1=T_{wc}^{(1)}(t)$ and $\Pi_2=T_{wc}^{(2)}(t)$ be \emph{any}
two camera poses at the same timestamp $t$ -- not necessarily related by a rotation, or by any other closed-form
relationship. By Eq.~\plaineqref{eq:projected_pixels}, the same primitive's center $\mathbf p(t)$ projects to generally different
pixels $\mathbf u_p^{(1)}(t)$ under $\Pi_1$ and $\mathbf u_p^{(2)}(t)$ under $\Pi_2$. What is invariant is not the pixel
location but what gets attached to it: the primitive's canonical anchor $\mathbf p_0$ and its identity noise vector
$\mathbf z_i$ of Eq.~\plaineqref{eq:primitive_noise} are attributes of the primitive itself, fixed once at construction time and
never a function of $t$, of $\Pi_1$, or of $\Pi_2$ (Sec.~\ref{sec:canonical_3d_noise}).

\begin{assumption}[Well-resolved primitive] \label{ass:well_resolved}
At pixel $p$ under $\Pi_1$, primitive $i$ dominates the rasterization weight: $w_{p,i}(\Pi_1)>0$ and $w_{p,j}(\Pi_1)\approx0$
for all $j\ne i$ (and symmetrically for pixel $p'$ under $\Pi_2$). This is the idealization underlying the statement in
Sec.~\ref{sec:view_consistent_whitening} that correspondence is exact when the same primitives dominate the rasterization;
it fails at genuine multi-primitive blending, most notably at occlusion boundaries.
\end{assumption}

\begin{proposition}[View consistency of the rasterized noise] \label{prop:view_consistency}
Let $p$ be the pixel under $\Pi_1$ and $p'$ the pixel under $\Pi_2$ at which primitive $i$ is rasterized (i.e.\
$p=\mathbf u_p^{(1)}(t)$, $p'=\mathbf u_p^{(2)}(t)$ from Eq.~\plaineqref{eq:projected_pixels}). Under
Assumption~\ref{ass:well_resolved},
\begin{equation}
\mathbf X_p(\Pi_1) = \mathbf X_{p'}(\Pi_2) = \mathbf p_0, \qquad
\mathbf Q_p(\Pi_1) = \mathbf Q_{p'}(\Pi_2), \qquad
\boldsymbol\epsilon_p^{\mathrm{3D}}(\Pi_1) = \boldsymbol\epsilon_{p'}^{\mathrm{3D}}(\Pi_2),
\label{eq:view_consistency_result}
\end{equation}
for \emph{any} two camera poses $\Pi_1,\Pi_2$ at the same timestamp $t$ -- no rotation, translation, or other structural
relationship between $\Pi_1$ and $\Pi_2$ is required.
\end{proposition}
\begin{proof}
Under Assumption~\ref{ass:well_resolved}, the canonical-position rasterization of Eq.~\plaineqref{eq:canonical_position}
collapses to the single dominant term: $\mathbf X_p(\Pi_1) = \frac{w_{p,i}(\Pi_1)\mathbf p_0}{w_{p,i}(\Pi_1)+\varepsilon}
\to\mathbf p_0$ as $w_{p,i}(\Pi_1)\gg\varepsilon$, and identically $\mathbf X_{p'}(\Pi_2)\to\mathbf p_0$, giving the first
equality in Eq.~\plaineqref{eq:view_consistency_result}; this uses only that $\mathbf p_0$ does not depend on $\Pi_1$ or $\Pi_2$.
Since $\mathbf Q_p$ is a fixed deterministic function of the canonical position alone (Eq.~\plaineqref{eq:projection_qr}), equal
anchors give $\mathbf Q_p(\Pi_1)=\mathbf Q(\mathbf p_0)=\mathbf Q_{p'}(\Pi_2)$. Similarly, the normalized feature of
Eq.~\plaineqref{eq:normalized_noise_feature} collapses to $\widetilde{\mathbf Y}_p(\Pi_1) = \frac{w_{p,i}(\Pi_1)\mathbf
z_i}{\sqrt{w_{p,i}(\Pi_1)^2+\varepsilon}}\to\mathbf z_i$ (using $w_{p,i}(\Pi_1)>0$), independent of the specific value of
$w_{p,i}(\Pi_1)$, and identically $\widetilde{\mathbf Y}_{p'}(\Pi_2)\to\mathbf z_i$; this uses only that $\mathbf z_i$ does not
depend on $\Pi_1$ or $\Pi_2$ either. Combining the equal $\mathbf Q$ and $\widetilde{\mathbf Y}$ through
Eq.~\plaineqref{eq:final_structured_noise} gives $\boldsymbol\epsilon_p^{\mathrm{3D}}(\Pi_1) = \mathbf Q(\mathbf p_0)^\top\mathbf
z_i = \boldsymbol\epsilon_{p'}^{\mathrm{3D}}(\Pi_2)$, the last equality of Eq.~\plaineqref{eq:view_consistency_result}. At no point
did the proof use a relationship between $\Pi_1$ and $\Pi_2$, a closed-form pixel correspondence, or the specific values of
$\mathbf u_p^{(1)}(t),\mathbf u_p^{(2)}(t)$ from Eq.~\plaineqref{eq:projected_pixels} beyond their existence.
\end{proof}

\begin{corollary}[Pure rotation gives the correspondence in closed form] \label{cor:homography}
If $\Pi_2$ is obtained from $\Pi_1$ by a pure rotation $\mathbf R\in SO(3)$ about its own optical center (same $\mathbf c(t)$,
$R_{cw}^{(2)}(t)=\mathbf R R_{cw}^{(1)}(t)$, i.e.\ $R_{wc}^{(2)}(t)=R_{wc}^{(1)}(t)\mathbf R^\top$, with $\mathbf R$ expressed in
the first camera's frame), the pixel $p'$ at which Prop.~\ref{prop:view_consistency} applies is additionally
known in closed form, independent of $\mathbf p_0$'s depth: writing homogeneous pixel coordinates $\widetilde{\mathbf u}$,
\begin{equation}
\widetilde{\mathbf u}' \propto \big(\mathbf K\mathbf R\mathbf K^{-1}\big)\widetilde{\mathbf u}.
\label{eq:homography_corollary}
\end{equation}
\end{corollary}
\begin{proof}
Writing $\mathbf x^c=R_{cw}^{(1)}(t)(\mathbf p(t)-\mathbf c(t))$ and $z=[\mathbf x^c]_3>0$, so $\mathbf x^c=z\mathbf
K^{-1}\widetilde{\mathbf u}$, the rotated camera's coordinate is $\mathbf x^{c\prime}=R_{cw}^{(2)}(t)(\mathbf p(t)-\mathbf c(t))=\mathbf R\mathbf x^c=z\mathbf R\mathbf
K^{-1}\widetilde{\mathbf u}$, so $\widetilde{\mathbf u}'\propto\mathbf K\mathbf x^{c\prime}=z\,\mathbf K\mathbf R\mathbf
K^{-1}\widetilde{\mathbf u}\propto\mathbf K\mathbf R\mathbf K^{-1}\widetilde{\mathbf u}$, the depth $z$ being absorbed by the
projective equivalence. Unlike Prop.~\ref{prop:view_consistency}, which needs no relationship between $\Pi_1,\Pi_2$ because it
never has to locate $p'$ explicitly, this closed form exists only because a pure rotation about the optical center introduces
no parallax (Hartley \& Zisserman~\cite{hartley2003multiple}); a general pair $\Pi_1,\Pi_2$ (translation included) has no such
depth-independent map, which is exactly why Prop.~\ref{prop:view_consistency} is stated without one.
\end{proof}

\begin{remark}
Two points on scope. \emph{(i) Where Assumption~\ref{ass:well_resolved} fails:} at genuine multi-primitive blending -- most notably a pixel
straddling an occlusion boundary -- $\mathbf X_p$ averages the anchors of unrelated surfaces (Eq.~\plaineqref{eq:canonical_position}),
so the clean collapse to $\mathbf p_0$ in the proof no longer holds exactly; this is the same exception already noted in
Sec.~\ref{sec:view_consistent_whitening} and in Prop.~\ref{prop:decorrelation}'s remark. \emph{(ii) Locus of the guarantee:}
as with Props.~\ref{prop:decorrelation}, this is a statement about the \emph{noise construction}, exact under
Assumption~\ref{ass:well_resolved}; it does not by itself imply the trained diffusion model's generated frames inherit this
consistency, which is what the view-synchronization metrics of Sec.~\ref{sec:metrics} test empirically.
\end{remark}

\subsection{Experimental Details} \label{sec:supp_experimental_details}
\textbf{Training.} We evaluate all baseline methods using their officially released
checkpoints. We train \name with the Wan2.2-14B backbone on eight H100 GPUs,
using a per-GPU batch size of one and gradient accumulation over four steps,
for an effective batch size of 32. The model is trained for 10K iterations with
a learning rate of $1 \times 10^{-4}$.

\textbf{4D representation.} We use an off-the-shelf feed-forward Gaussian Splatting
model, NoPo4D~\cite{balice2026nopo4d}, which avoids per-video reconstruction optimization and produces dynamic 
primitives that are active over finite temporal windows. Each primitive has a time-varying center,
covariance, opacity, and appearance, together with a time-invariant canonical
anchor. To render noise for views not covered by the source video, we model
unobserved background regions with an enclosing cubemap: a finite,
axis-aligned box centered at the centroid of the canonical anchors, with a
half-extent set to a fixed multiple of the reconstructed scene extent. The
cubemap therefore lies at a depth comparable to the scene rather than at
infinity, as in an environment map. We index background texels by the
ray--box intersection point instead of ray direction alone, allowing pixels that
observe the same proxy surface to share a canonical location. The cubemap has
six square faces, each containing $S\times S$ texels with $d$ channels; we use
$S=1024$ and $d=3$.

\textbf{Inference.} Unless stated otherwise, we use 50 denoising steps for all
\name experiments, matching the default inference configuration of the WAN
video models. We use the same fixed template as the text prompt during training
and inference. Table~\ref{tab:compute_time_breakdown} reports the runtime of
the main \name components. Constructing the 4D representation takes 4.18 seconds
and is inexpensive relative to total generation time; because it is performed
once per source video, this cost can also be amortized across multiple reshoots.
The reported noise-rasterization time was measured on a particularly large
dynamic scene containing approximately 6.53M Gaussian points. Video generation
uses the original WAN configuration. More efficient step-distillation methods
could reduce this cost, but evaluating them is beyond the scope of this work.

\textbf{Metrics.} For pixel matching (Pix.~Mat.), we use RoMA-v2~\cite{edstedt2025romav2}
to estimate point correspondences and count matches with uncertainty below
$0.5$. We report the average number of matched pixels separately for dynamic
and static regions. To identify dynamic foreground objects, we obtain a
segmentation mask from the source video using SAM3~\cite{carion2026sam3}.
We use this mask to partition correspondences for each source--target frame
pair, enabling separate evaluation of dynamic and static regions.

\begin{table}[t]
\centering
\caption{\textbf{Compute time breakdown} for the four main components of \name. Times are reported in seconds per input clip with 75 frames at 480$\times$832 resolution on an Nvidia H100 GPU. The DiT time is reported per denoising step.}
\label{tab:compute_time_breakdown}
\resizebox{0.9\textwidth}{!}{%
\begin{tabular}{lcccccc}
\toprule
 & \shortstack{4D representation\\(NoPo4D)} & \shortstack{Noise\\rasterization} & Whitening & \multicolumn{3}{c}{\shortstack{Video generation\\(WAN2.2-14B)}} \\
\midrule
 & & & & DiT & VAE decoder & Total \\
\cline{5-7} \\
Compute time (sec.) & 4.18 & 6.44 & 0.53 & 23.97 / step & 4.60 & 1203.1 \\
\bottomrule
\end{tabular}%
}
\end{table}

\subsection{Qualitative Comparison} \label{sec:supp_qualitative_comparison}

For a more detailed qualitative comparison, we encourage readers to consult the HTML viewer
provided in the supplemental materials. In the print version, Fig.~\ref{fig:supp_qual_comparison}
shows additional results from test scenes under a Dolly-out camera trajectory, while
Figs.~\ref{fig:supp_qual_comparison_arc_left} and~\ref{fig:supp_qual_comparison_arc_right}
show results under Arc-left and Arc-right trajectories, respectively. Each example
contains six equally spaced frames from the source clip and from each method's generated reshoot.

\begin{figure}[t]
\centering
\includegraphics[width=\textwidth]{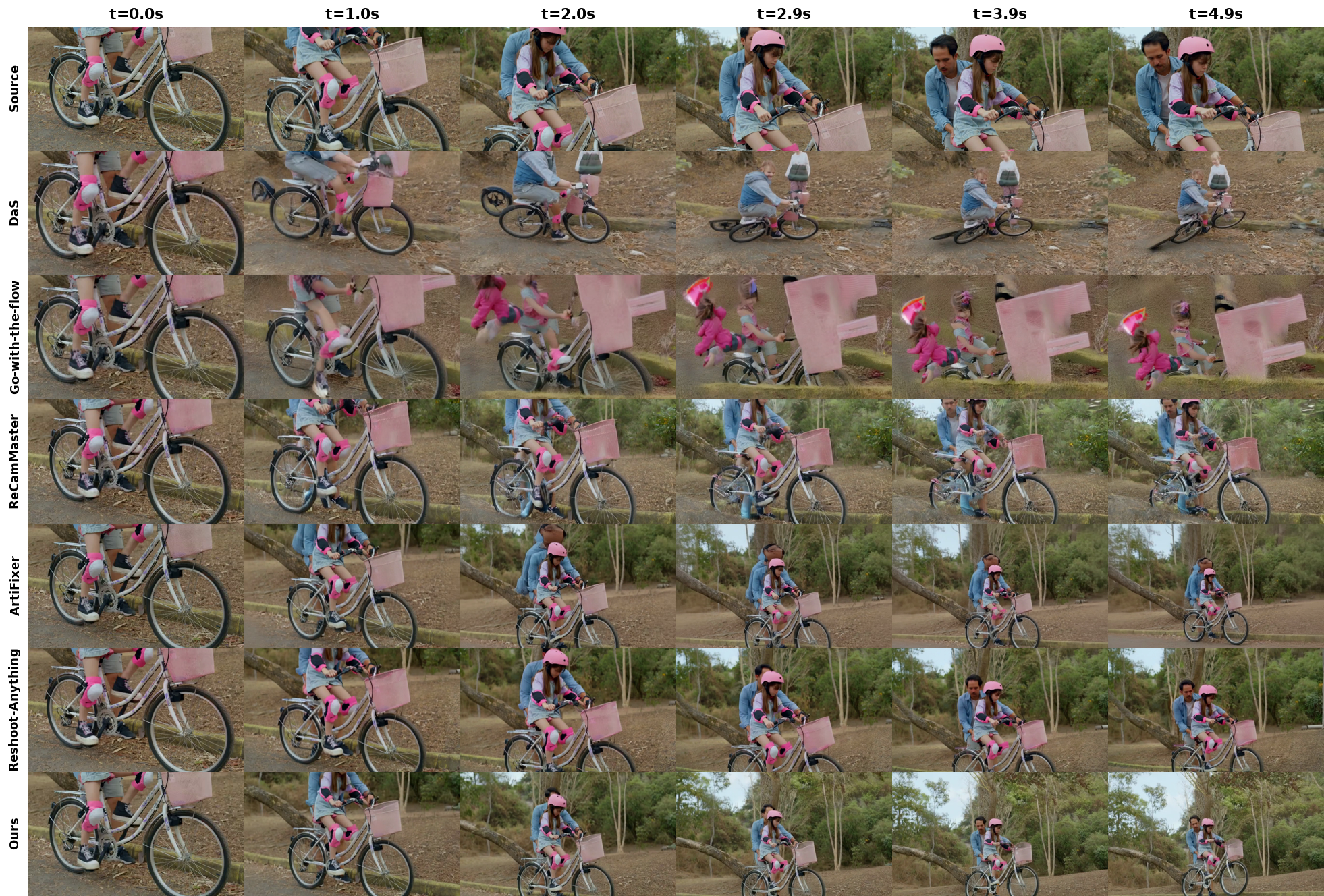}
\vspace{4pt}
\includegraphics[width=\textwidth]{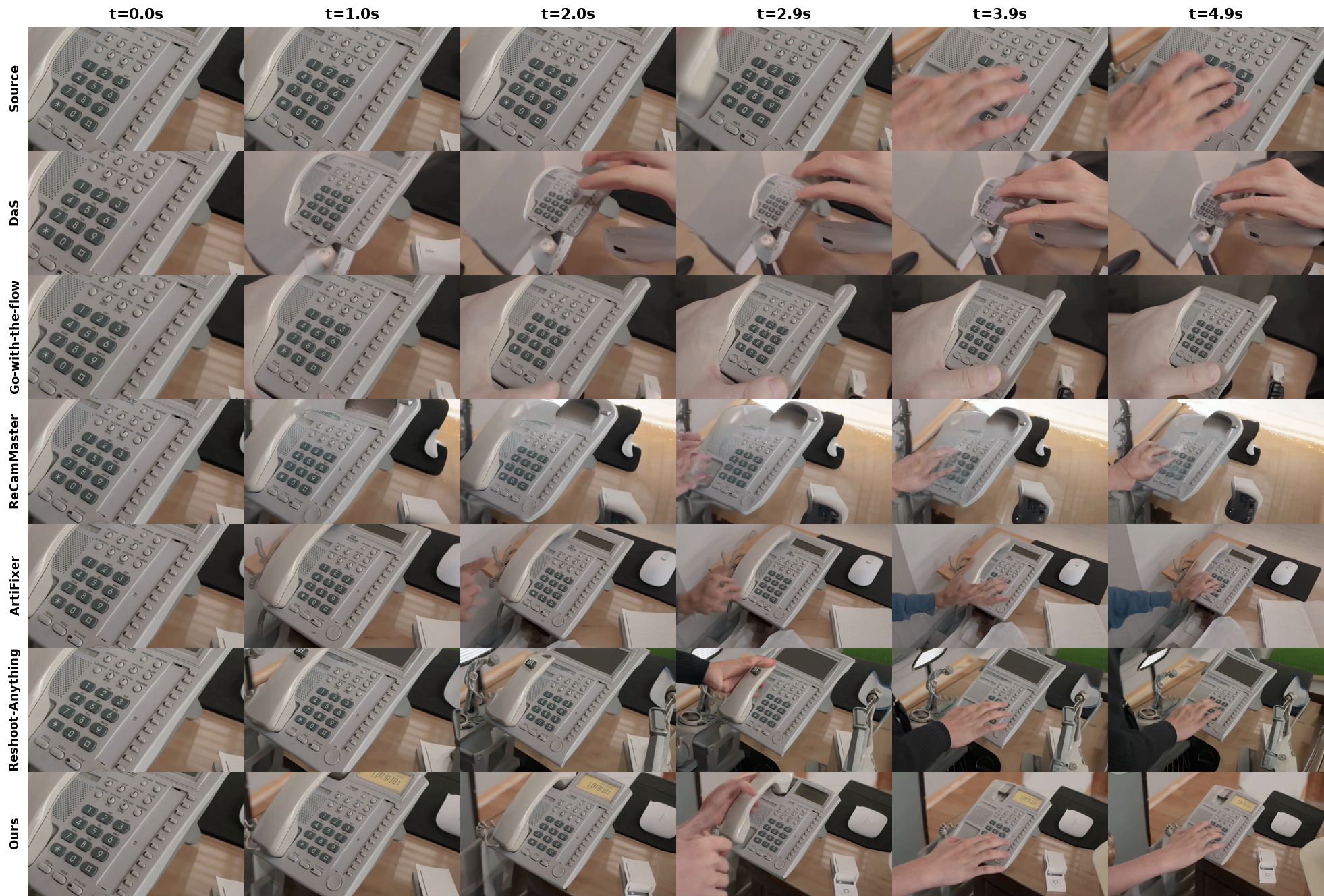}
\caption{\textbf{Qualitative comparisons} under a Dolly-out camera trajectory. For each example,
six equally spaced frames are shown for the source video and each method's generated reshoot.}
\label{fig:supp_qual_comparison}
\end{figure}

\begin{figure}[t]
\centering
\includegraphics[width=\textwidth]{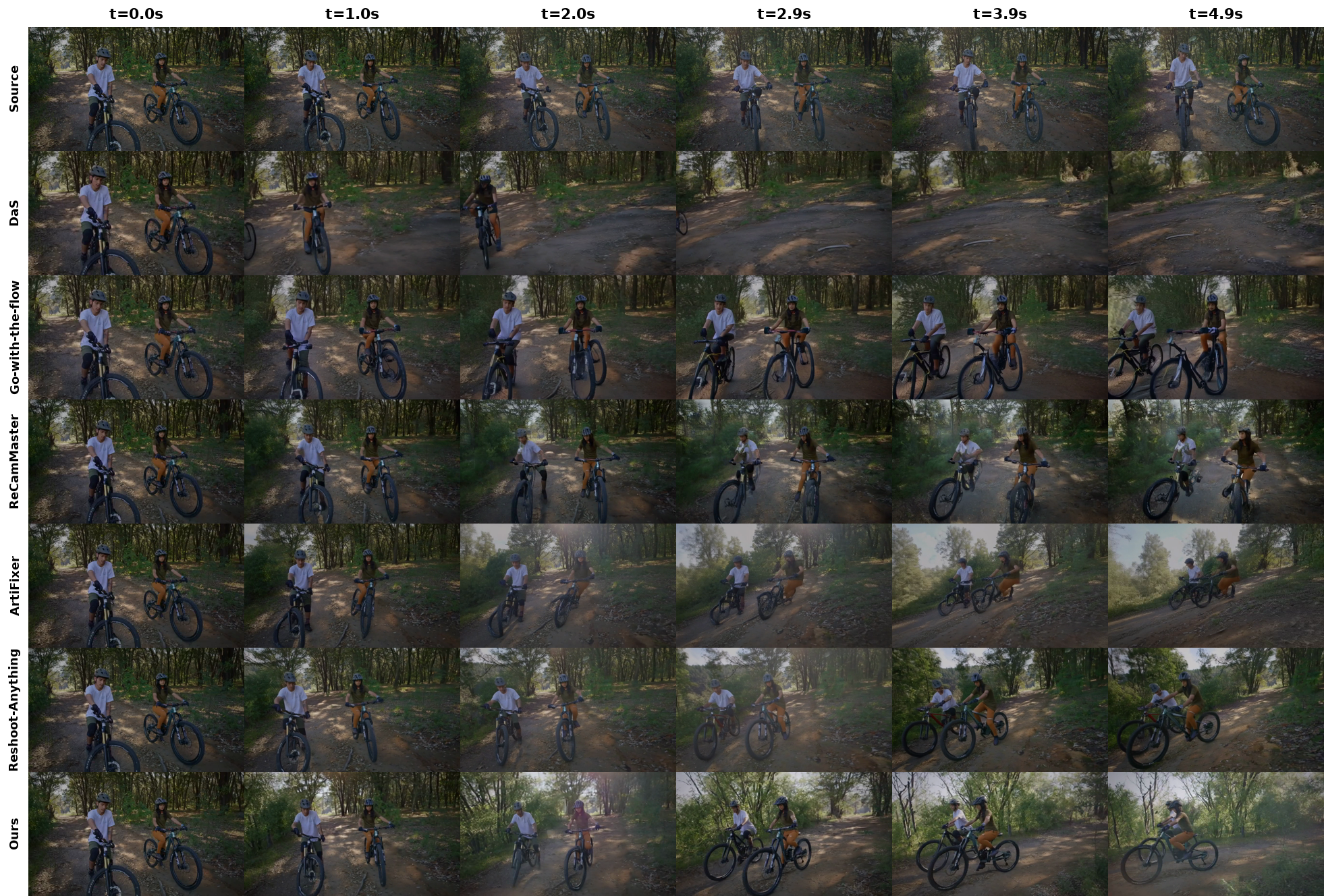}
\vspace{4pt}
\includegraphics[width=\textwidth]{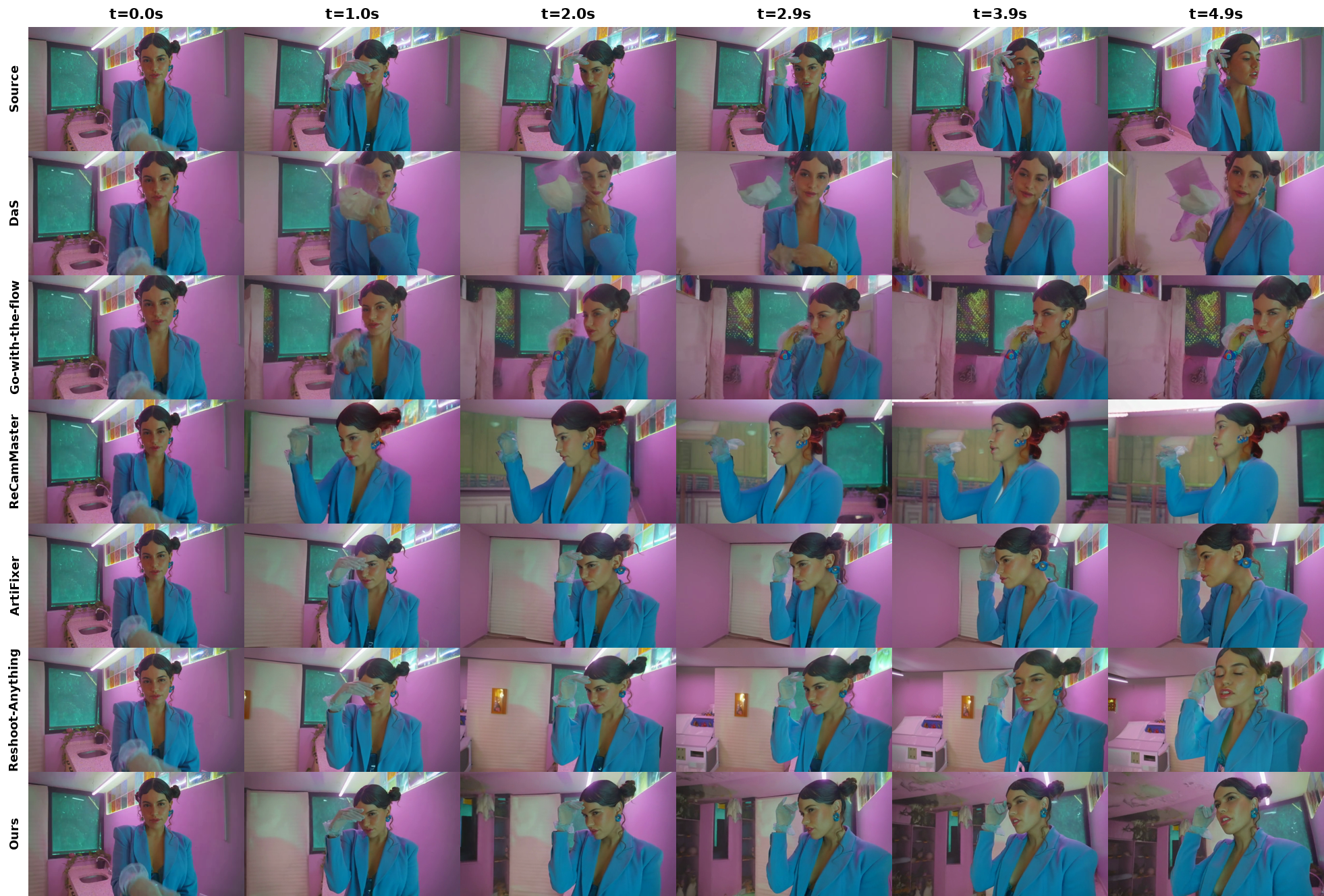}
\caption{\textbf{Qualitative comparisons} under an Arc-right camera trajectory. For each example,
six equally spaced frames are shown for the source video and each method's generated reshoot.}
\label{fig:supp_qual_comparison_arc_left}
\end{figure}

\begin{figure}[t]
\centering
\includegraphics[width=\textwidth]{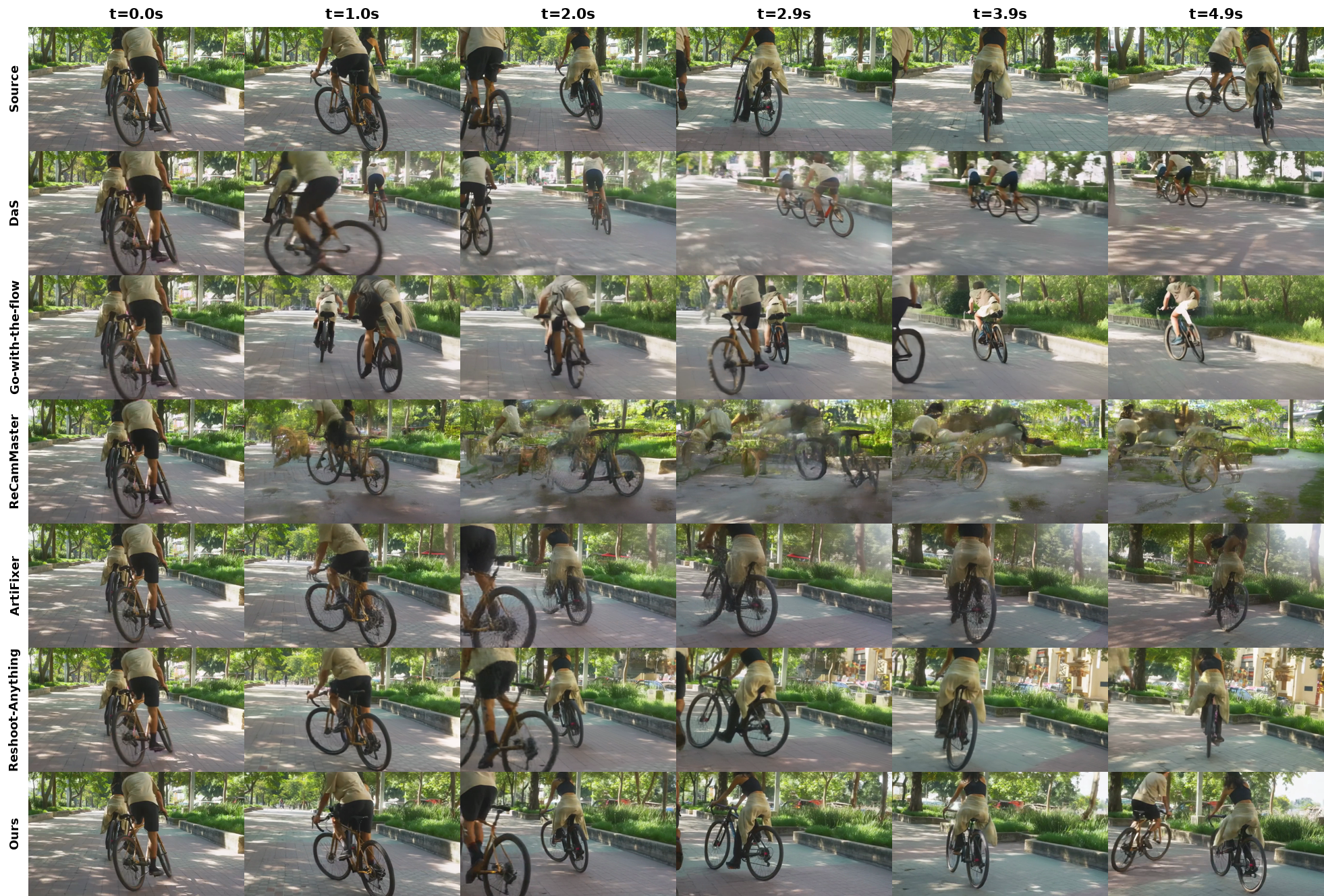}
\vspace{4pt}
\includegraphics[width=\textwidth]{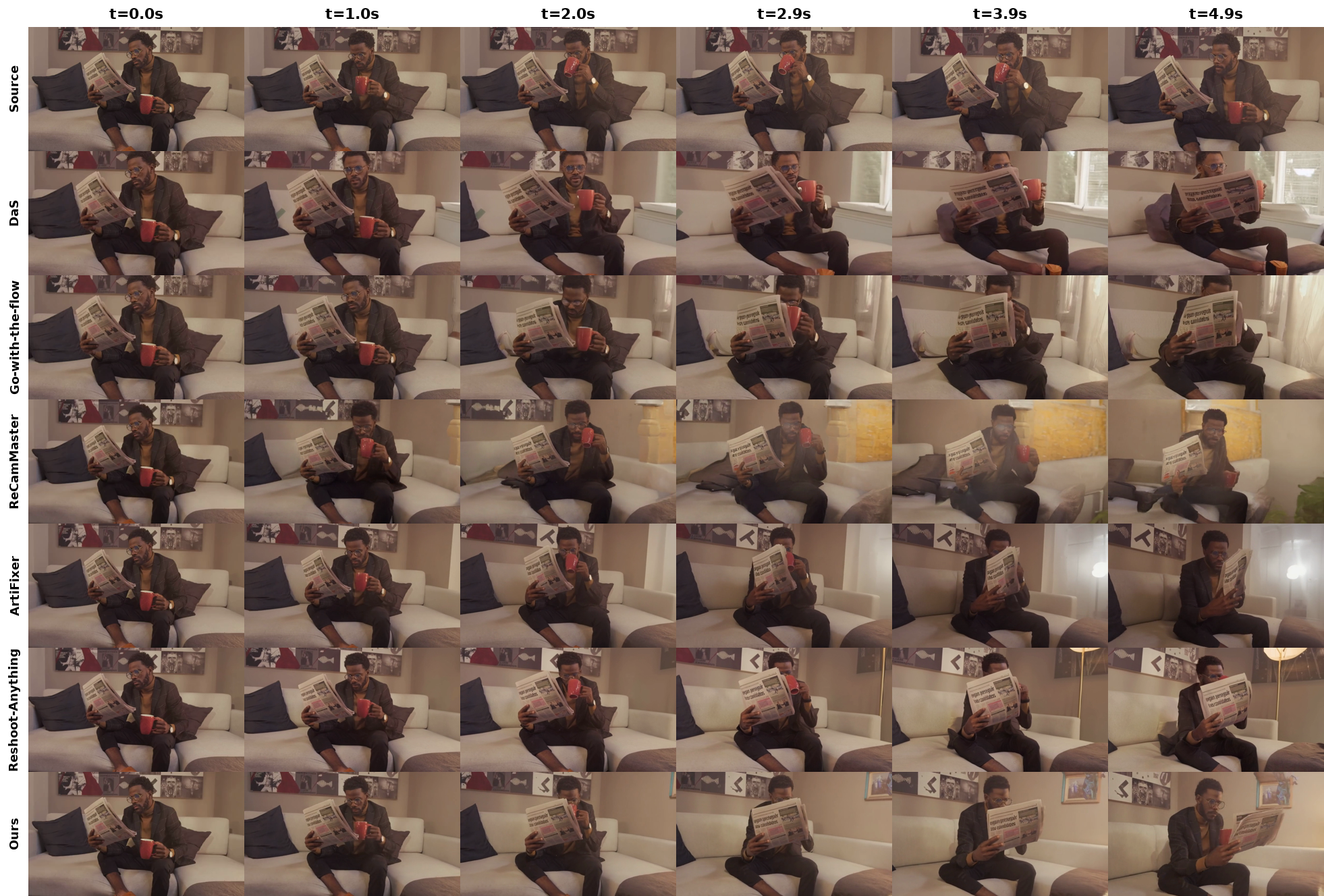}
\caption{\textbf{Qualitative comparisons} under an Arc-left camera trajectory. For each example,
six equally spaced frames are shown for the source video and each method's generated reshoot.}
\label{fig:supp_qual_comparison_arc_right}
\end{figure}

\subsection{Limitations and Failure Examples} \label{sec:supp_failure_examples}
Although SyncVid is trained with varied errors in the 4D noise representation
and is therefore less sensitive to reconstruction quality---the 4D
representation is used primarily to rasterize a synchronized noise prior---it
remains limited by the quality of the monocular reconstruction. Errors in
geometry, tracking, or primitive re-instantiation can still appear as subject
deformation or temporal drift, while disoccluded regions are only weakly
constrained and may be hallucinated by the diffusion model. Figure~\ref{fig:supp_failure_examples}
illustrates these failure modes: the Arc-right example exhibits severe
foreground inconsistencies, whereas the Dolly-out example retains the main
subject and setting but introduces unrealistic content in newly exposed regions.

Although SyncVid excels on video-synchronization metrics and in preserving
dynamic regions in the generated videos, it does not outperform state-of-the-art
methods in camera-pose estimation; methods such as ArtiFixer perform well on this
task. We speculate that this gap is partly due to the absence of explicit
camera-pose conditioning, such as the Pl\"ucker embeddings used by ArtiFixer.
Providing this additional conditioning could improve pose accuracy, but
integrating it is orthogonal to the structured-noise formulation presented here
and is left for future work.

\begin{figure}[t]
\centering
\includegraphics[width=\textwidth]{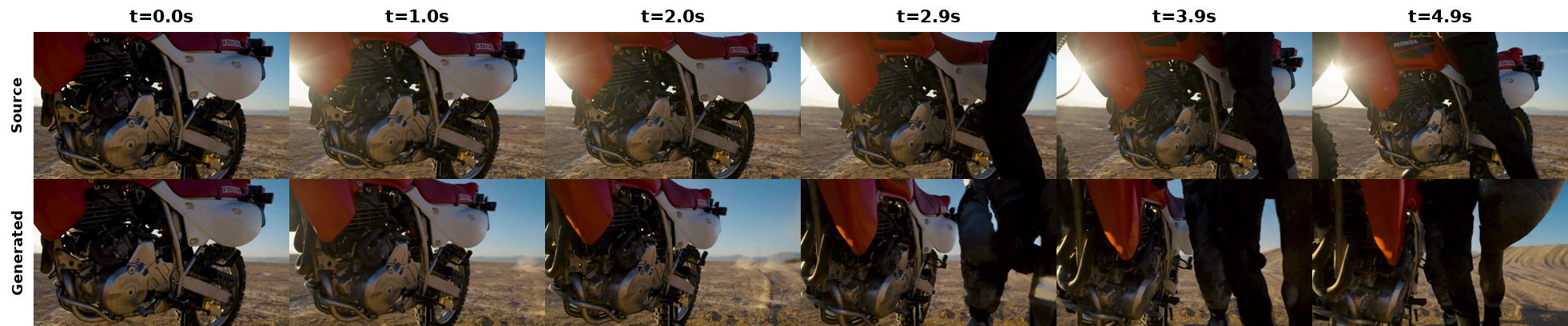}
\vspace{4pt}
\includegraphics[width=\textwidth]{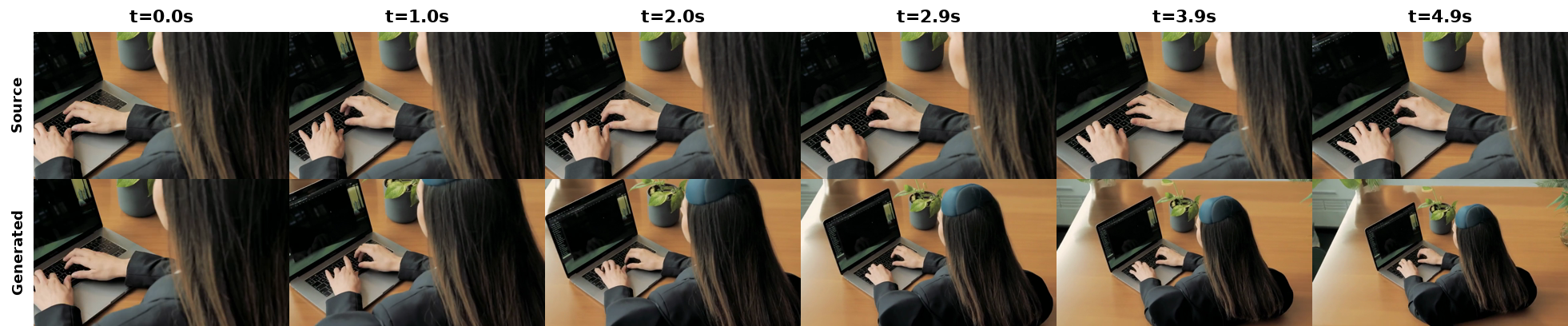}
\caption{\textbf{Failure examples.} Each example shows six equally spaced source frames in the
first row, followed by the corresponding generated frames. 
The top example follows Arc-left and the bottom example shows Dolly-out trajectory.}
\label{fig:supp_failure_examples}
\end{figure}

\end{document}